\documentclass[11pt]{article}
\PassOptionsToPackage{hyphens}{url}
\usepackage[preprint]{acl}
\usepackage{times}
\usepackage{latexsym}
\usepackage[T1]{fontenc}
\usepackage[utf8]{inputenc}
\usepackage{microtype}
\usepackage{graphicx}
\usepackage{xcolor}
\usepackage{url}
\usepackage{booktabs}
\usepackage{array}
\usepackage{amsmath,amssymb,amsthm}
\usepackage{algorithm}
\usepackage{algpseudocode}
\usepackage{tikz}
\usetikzlibrary{arrows.meta,positioning}
\usepackage{pgfplots}
\usepackage{placeins}
\pgfplotsset{compat=1.17}
\newcommand{\E}{\mathbb{E}}
\newcommand{\Prb}{\mathbb{P}}
\newcommand{\calM}{\mathcal{M}}
\newcommand{\calS}{\mathcal{S}}
\newcommand{\calA}{\mathcal{A}}
\newcommand{\calH}{\mathcal{H}}
\newcommand{\calY}{\mathcal{Y}}

\newcommand{\Reg}{\operatorname{Regret}}
\newcommand{\PSRL}{\operatorname{PSRL}}

\newcommand{\Gap}{\operatorname{Gap}}

\theoremstyle{plain}
\newtheorem{theorem}{Theorem}
\newtheorem{lemma}{Lemma}
\newtheorem{proposition}{Proposition}
\newtheorem{corollary}{Corollary}
\newtheorem{assumption}{Assumption}
\theoremstyle{definition}

\title{When Is Language-Model Advice Worth Buying?\\
A Matched-Budget Study of Predictive Gating in Reinforcement Learning}

\author{
 \textbf{Ibne Farabi Shihab\textsuperscript{1}},
 \textbf{Md Najmus Swaqeeb\textsuperscript{2}},\\
 \textbf{Abu Sa-Adat Mohamed Moon-Im Al Ahsan\textsuperscript{3}},
 \textbf{Joyanta Jyoti Mondal\textsuperscript{4}}
\\
\\
 \textsuperscript{1}Department of Computer Science, Iowa State University, USA\\
 \textsuperscript{2}Standard Chartered Bangladesh\\
 \textsuperscript{3}Department of Computer Science and Engineering, BRAC University, Bangladesh\\
 \textsuperscript{4}Department of Computer and Information Sciences, University of Delaware, USA
\\
 \small{
   \textbf{Correspondence:} \href{mailto:ibnfarabishihab@gmail.com}{ibnfarabishihab@gmail.com}
 }
}

\begin{document}
\raggedbottom
\maketitle

\begin{abstract}
Language-model advice can accelerate reinforcement learning, but every call has a price, and a returned suggestion may be unparsable, unsafe to execute, or no better than the agent's own action. We introduce certified predictive advice gating, a cost-aware controller that forecasts possible advisor responses, values the action actually executed after parsing, execution checks, and cache reuse, and requests advice only when a calibrated lower bound on that value covers the price. Our analysis separates the gap of executed advice, cumulative regret under an intervention-stable fallback, and myopic acquisition loss. On BabyAI GoToObj with Qwen2.5-1.5B and 7B advisors, the gate improves return over no advice by paired $0.029\pm0.016$ and $0.030\pm0.015$ while using more than $97\%$ fewer calls than always asking. Matched-budget comparisons, however, do not resolve an advantage over an early query schedule, so the gain reflects economical advice volume rather than demonstrated state selection. Groupwise calibration raises coverage on positive-value opportunities from $0.47$ to $0.85$, still below the $0.90$ target. A separate headroom-aware gate reports paired gains over an early schedule of $0.032\pm0.012$ and $0.029\pm0.011$ on GoToObj and $0.027\pm0.011$ and $0.028\pm0.012$ on GoToObjDoor, and adds BabyAI and ALFWorld success, decision-stratum coverage, and component timing. These results separate return gains, query placement, and execution screening; finite-sample deployment guarantees remain unestablished.
\end{abstract}

\section{Introduction}

In this paper, we study when a reinforcement-learning (RL) agent acquires advice from a large language model (LLM) and when it executes the suggestion. Queries add cost and latency; poor suggestions can reduce return. Acquisition depends on the expected improvement from a response, while execution depends on the suggested action's value and uncertainty relative to the fallback policy.

In BabyAI \citep{chevalierboisvert2019babyai}, the agent follows a language instruction from a partial grid view, selecting among seven primitive actions under sparse reward. An LLM can suggest an action but may misread the scene, return invalid text, or repeat an unhelpful suggestion. Query value must account for the RL policy and cached advice available without a fresh response.

We examine a response-contingent gate that forecasts parsed responses, values the action selected under each response, and queries when a lower estimate of improvement exceeds the price. Rejected advice uses the fallback; cached advice remains subject to execution checks. Figure~\ref{main:architecture} summarizes the sequence, and Algorithm~\ref{alg:voi} gives the controller. The acquisition decision precedes the fresh response.

The central empirical question is whether
predictive gating improves query placement
rather than merely regulating query volume.
We compare the controller with an early
schedule that spends its assigned calls at
the earliest eligible decisions, subject
to the same per-episode query cap.
The schedule receives the predictive gate's
realized per-seed call allocation and therefore
serves as a retrospective placement control,
not a deployable budget-selection policy.
We separately examine a headroom-aware variant
designed to target advice that changes the
executed action or continuation.

Our contributions are:
\begin{itemize}
\setlength{\itemsep}{0pt}
\item A response-contingent acquisition rule that values the action deployed after parsing, execution checks, and cache reuse, and subtracts a split-conformal radius on query-value error before comparing with the price (Section~\ref{sec:gating}).
\item Guarantees that separate the gap of executed advice, cumulative regret under an intervention-stable fallback, and myopic acquisition loss, together with the assumptions required to apply them to the deep controller (Section~\ref{sec:guarantees}).
\item A real-advisor BabyAI evaluation in which the original controller improves GoToObj return with Qwen2.5-1.5B and 7B advice while reducing calls by more than $97\%$ relative to always asking; it shows no resolved matched-budget advantage over an early schedule, and its empirical radius undercovers opportunities with positive estimated query value (Section~\ref{sec:results}).
\item A separately evaluated headroom-aware extension that targets decision-changing advice, where headroom is the potential improvement over the fallback without advice; it reports positive paired differences over an early schedule on GoToObj and GoToObjDoor without adjusted tests (Section~\ref{sec:headroom-extension}).
\end{itemize}


\begin{figure*}[t]
\centering
\includegraphics[width=\textwidth]{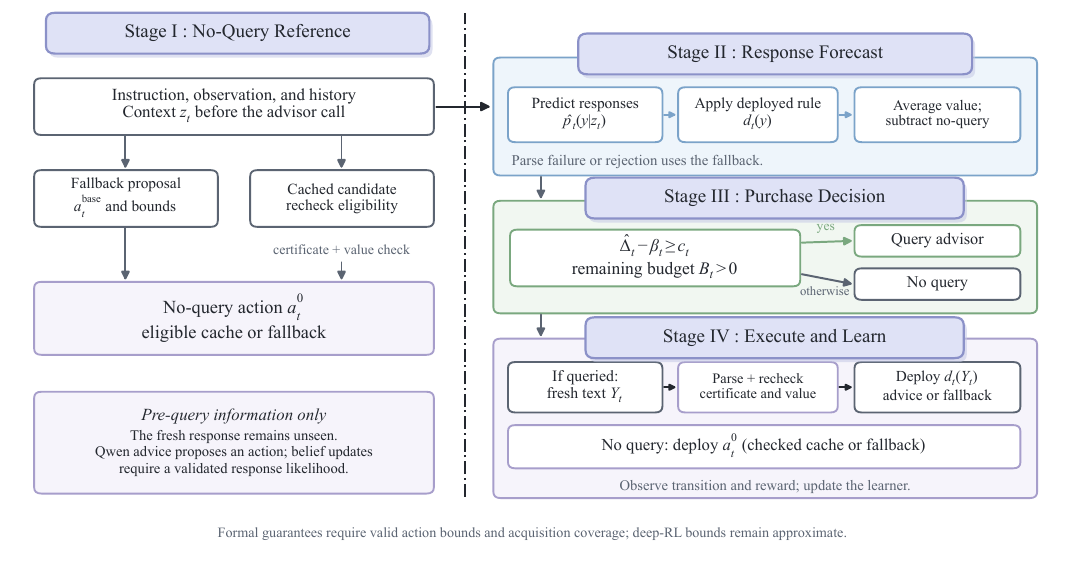}
\caption{Predictive advice gating. I: the controller establishes a no-query reference from the fallback and eligible cached advice. II: it forecasts responses and values each deployed decision, including fallback after parse failure or rejection. III: it purchases advice only when the lower predicted gain covers the price and budget remains. IV: it parses and rechecks fresh advice before execution, or executes the no-query action when no call occurs. The resulting transition and reward update the learner.}
\label{main:architecture}
\end{figure*}

\section{Related Work}
\label{sec:related}

Our work builds on asking-for-help and teacher-advising policies \citep{torrey2013teaching,dasilva2020uncertainty,liu2022askingforknowledge,hu2023when2ask,monteiro2026ask}, cost-aware prediction, routing, and cascades \citep{chen2020frugalml,ong2024routellm,jitkrittum2023cascade}, and distribution-free calibration, including conformal planners that ask for help \citep{angelopoulos2023gentle,ren2023knowno}. Unlike the uncertainty-threshold ASK gate \citep{monteiro2026ask}, our formulation averages over predicted responses through the deployed action rule, accounting for parsing failure, rejection, and cache reuse; our certificate concerns the value of purchasing advice whose effect propagates through an RL fallback. Learned consultation and CCPO \citep{feng2024agile,si2026ccpo} use different action and training interfaces, so their published scores do not provide matched baselines, and our comparisons do not establish superiority to them. Appendix~\ref{app:related} gives a detailed discussion.

%
\section{Method: Certified Predictive Advice Gating}
\label{sec:gating}

This section defines the advice interface and execution check, the query-value estimate and its calibrated radius, and the controller that combines them. Table~\ref{tab:notation} in Appendix~\ref{app:notation} collects the notation, and Appendix~\ref{app:full-controller} gives the full formulation.

\subsection{Problem Formulation}

An episodic Markov decision process (MDP) consists of states $\calS$, actions $\calA$, transition law $P$, reward $r$, and horizon $H$: $\calM=(\calS,\calA,P,r,H)$. At decision $t$, the agent has history $\calH_t$, state or representation $s_t$, and fallback proposal $a_t^{\rm base}$. The advisor receives context $z_t=\phi(\calH_t)$, including the instruction and observable state, and returns text $Y_t$. The fixed parser $g(Y_t,z_t)$ maps the response to an action in $\calA$ or to the failure symbol $\bot$.

Each query consumes one unit of hard budget $B$, with remaining balance $B_t$. The purchase threshold is $c_t=\lambda\kappa_t$, where $\kappa_t$ is the declared query price and $\lambda$ is a budget multiplier. Before calling, $\widehat p_t(y\mid z_t)$ predicts parsed responses. BabyAI has $|\calA|=7$ actions; the response distribution also includes $\bot$.

Acquisition values behavior under each possible response. Let $Q^y_{M,t}(s,a)$ denote expected return under a response-dependent continuation fixed before observing $Y_t$; $Q^0_{M,t}$ denotes the no-query branch. The window is $\ell_t\le H$ (six steps in the original BabyAI controller). Write $\mu^y_{b,t}=\E_{M\sim b}[Q^y_{M,t}]$, with $b$ a belief over environments or a value-model approximation.

The role of the response depends on the advisor model. It is \emph{informative} only when a validated likelihood $p(y\mid M,z_t)$ supports a Bayesian update to $b_t^y$. In the Qwen experiments, advice is \emph{candidate-only}: it proposes an action while leaving $b_t$ unchanged. Its predictive value is therefore distinct from Bayesian information gain.

The execution check uses a different target, the optimal action value $Q_h^\star$. Let $L_t(s,a)$ and $U_t(s,a)$ be lower and upper bounds on that quantity. An advised action $a_y$ is eligible for execution only if
\begin{equation}
L_t(s_t,a_y)\ge\max_{a\in\calA}U_t(s_t,a)-\epsilon_t.
\label{main:safe}
\end{equation}
The certificate uses the declared gap tolerance. The controller executes certified advice only when its continuation value is at least the fallback's:
\begin{equation}
 d_t(y)=\begin{cases}
 a_y,&\substack{g(y,z_t)=a_y\ne\bot,\ \text{Eq.~\eqref{main:safe} holds},\\
 \mu_{b_t^y,t}^y(s_t,a_y)\ge\mu_{b_t^y,t}^y(s_t,a_t^{{\rm base},y})},\\
 a_t^{{\rm base},y},&\text{otherwise.}
 \end{cases}
 \label{main:decide}
\end{equation}
The no-query action $a_t^0$ uses a certified cached candidate only when its continuation value under $\mu^0_{b_t,t}$ is at least the fallback's; otherwise it uses $a_t^{\rm base}$. An entry is disabled if certification fails, the context falls outside calibrated predictor support, or a prespecified shift alarm fires. Eligibility alone does not establish safety.


\subsection{Query Value and Its Uncertainty}

Query value is the expected continuation gain over the fixed no-query branch. With $K$ sampled continuation values per response, the estimator averages deployed decisions over predicted responses:
\begin{equation}
\begin{aligned}
\widehat\Delta_t
&=\sum_y\widehat p_t(y\mid z_t)\frac1K\sum_{k=1}^K
 Q_{k,t}^{y}(s_t,d_t(y))\\
&\quad-\frac1K\sum_{k=1}^K Q_{k,t}^{0}(s_t,a_t^0).
\end{aligned}
\label{main:est}
\end{equation}
The outer sum includes parse failures and rejected suggestions; inner averages use posterior or ensemble continuation samples. Informative advice requires response-conditioned samples. Unconditional samples are valid only for candidate-only advice.

If the fallback already selects the best action under a fully known $Q$ and the continuation is unchanged, candidate-only advice has nonpositive decision value. A positive effect must instead arise from an imperfect or constrained learner, an altered continuation, or information unavailable to the learner. Proposition~\ref{prop:candidate-only} in Appendix~\ref{app:theory} states this restriction formally.

We estimate acquisition error on randomized calibration opportunities disjoint from predictor fitting and evaluation. The learner, predictor, cache, and continuation map remain frozen at each opportunity. We collect $J$ independent advisor-response blocks and $R$ paired simulator continuations per response, using common random numbers for query and no-query branches. Their return differences give the noisy proxy $\widetilde\Delta_i$. A Monte Carlo (MC) correction accounts for proxy noise. Let $D_i$ be a prespecified deterministic bound on the absolute paired-return difference. For $n$ calibration opportunities and failure probability $\delta_{\rm MC}$, the correction is
\begin{equation}
\rho_i^{\rm MC}=D_i\sqrt{\frac{2\log(2n/\delta_{\rm MC})}{J}}.
\label{main:mc}
\end{equation}
A split-conformal radius \citep{angelopoulos2023gentle} is computed from $|\widehat\Delta_i-\widetilde\Delta_i|+\rho_i^{\rm MC}$, optionally normalized by a separately fitted positive scale. Under the stated exchangeability and independence conditions, it covers error relative to latent query value. The independent units are response blocks, including correlated-call clusters where required, rather than rollouts that share a response.

For per-step rewards in $[0,1]$, $D_i=H$ is valid, and $D_i=\ell_i$ is valid for a declared $\ell_i$-step continuation (Appendix~\ref{app:bounded-and-audited}). A smaller bound requires independent justification from the reward structure; the observed sample range is insufficient. The reported BabyAI radii do not incorporate the local-range refinement. An empirical radius without the outer-block correction remains a diagnostic. The operational gate queries when its lower value estimate covers the price and budget remains:
\begin{equation}
q_t=\mathbf 1\{\widehat\Delta_t-\beta_t\ge c_t,\;B_t>0\}.
\label{main:gate}
\end{equation}
Calibration dominated by unchanged-action responses may poorly cover decision-changing opportunities. Chronological training and adaptive query selection can also violate exchangeability. We therefore audit decision-relevant and marginal coverage.
Appendix~\ref{app:example} works through a numerical purchase decision that accounts for rejected responses.


\subsection{Algorithm and Online Updates}

The response predictor uses the instruction, observation embedding, cache metadata, previous responses, parser outcomes, and past realized rewards. Only information available before the current query enters its forecast. The predictor updates after each paid response. The controller first fixes $a_t^0$, then forecasts responses and evaluates their deployed decisions (Algorithm~\ref{alg:voi}).

\begin{algorithm}[t]
\caption{Certified Predictive Value-of-Advice Gating}
\small
\label{alg:voi}
\begin{algorithmic}[1]
\State Input: fallback learner, response predictor, cache, budget $B$, thresholds $c_t$, tolerances $\epsilon_t$, radii $\beta_t$
\For{each episode}
  \State reset remaining budget; initialize or sample the fallback value model
  \For{each decision $t$}
    \State compute fallback proposal and action-value bounds; recheck cache and select $a_t^0$
    \State forecast responses and $d_t(y)$; compute $\widehat\Delta_t$ by Eq.~\eqref{main:est}
    \If{Eq.~\eqref{main:gate} holds}
      \State query and parse $Y_t$; update predictor and cache; decrement budget
      \State update belief only with a validated response likelihood
      \State recheck execution conditions; execute $d_t(Y_t)$ by Eq.~\eqref{main:decide}
    \Else
      \State execute $a_t^0$
    \EndIf
    \State observe transition and reward; update fallback and calibration statistics
  \EndFor
\EndFor
\end{algorithmic}
\end{algorithm}

\subsection{Headroom-Aware Variant}
\label{sec:headroom-method}

The headroom-aware variant targets advice that changes the executed action or continuation and uses separate calibration near the purchase threshold. Let $\widehat p_t^{\rm change}$ be the predicted probability of a change, $\widehat h_{G(t)}$ the nonnegative mean paired-return gain from training opportunities in task-advisor group $G(t)$, and $\beta_{g(t)}$ the empirical radius in pre-query stratum $g(t)$. The evaluated acquisition rule is

\begin{equation}
\begin{aligned}
q_t^{\rm HA}=\mathbf{1}\bigl\{&
\widehat p_t^{\rm change}\ge0.10,\\
&\widehat h_{G(t)}\ge2c_t,\\
&\widehat\Delta_t^{\rm change}
-\beta_{g(t)}\ge c_t,\\
&B_t>0
\bigr\}.
\end{aligned}
\label{eq:ha-gate}
\end{equation}

The change-probability cutoff is $0.10$, and the group-gain cutoff is twice the acquisition threshold. Appendix~\ref{app:ha-candidate} defines the response decomposition and group statistic. This variant is evaluated separately from the original controller.



\section{Theoretical Guarantees}
\label{sec:guarantees}

The theory separates three quantities, each under its own condition. \emph{Execution:} if $L_t(s,a)\le Q_h^\star(s,a)\le U_t(s,a)$ holds simultaneously at every visited pair except on a probability-$\delta$ failure event, Eq.~\eqref{main:safe} gives $\Gap_t(d_t(Y_t))\le\epsilon_t$ when advice is executed, where $\Gap_t(a)=V_h^\star(s_t)-Q_h^\star(s_t,a)$; it does not bound the gap of fallback actions. \emph{Regret:} if the fallback's Bellman gaps over non-advice steps on advice-induced histories sum to at most $\mathfrak R_{\rm base}(T)$, then $\E[\Reg(T)]\le\mathfrak R_{\rm base}(T)+\E[\sum_{t\in\cal C_T}\epsilon_t]+\delta TH$, where $\cal C_T$ contains executed-advice steps. \emph{Acquisition:} if $|\widehat\Delta_t-\Delta_t|\le\beta_t$ holds simultaneously at evaluated opportunities, every query triggered by Eq.~\eqref{main:gate} has nonnegative myopic net value, and the loss relative to a myopic oracle before budget exhaustion is at most $2\sum_t\beta_t$. These guarantees do not establish end-to-end return for the BabyAI controller, whose approximate deep action values and empirical query radius lack the required coverage (Appendix~\ref{sec:theory}).

\section{Experimental Setup}
\label{sec:setup}

We evaluate the original controller and headroom-aware extension as separate policies. The original study tests return, matched-budget placement, calibration, execution screening, and price response. The extension adds task-success, decision-stratum coverage, and component-timing measurements.

\subsection{Protocol and baselines}

The original controller uses fixed Qwen2.5-1.5B-Instruct and Qwen2.5-7B-Instruct advisors on BabyAI GoToObj (navigate to a named object) and GoToLocal (instruction-following navigation). The observation contains a $7\times7\times3$ grid image and direction encoding. The prompt supplies the instruction, observable state, admissible actions, and recent context; the advisor returns one of seven primitive actions, and invalid text maps to $\bot$. Methods share the replay-based fallback learner, warm start, environment steps, and paired seeds.

The original-controller evaluation uses 20 paired seeds and 60 episodes, horizon $H=64$, continuation window $\ell=6$, and at most $B=60$ queries per episode. Return is mean episodic undiscounted reward; call counts are summed across episodes and averaged over seeds. The original study lacks task-success logs. Success fractions belong to the separate extension, which also evaluates GoToObjDoor, Open, Unlock, PutNextLocal, and real-advisor ALFWorld pick-and-place and look-at-object tasks.

Baselines include never-query, always-query, the ASK-style uncertainty-only gate, and the predictive gate without calibration ($\beta{=}0$), with an empirical radius ($\beta_{\rm conf}$), or with the formal radius ($\beta_{\rm form}$). Random and early-schedule controls test placement at matched realized calls. Certificate-off and global-frequency or uniform predictor ablations test component effects. The ASK-style gate is an uncertainty baseline, not a matched AFK implementation. Appendix~\ref{sec:experiments} describes the environments, advisor protocols, and baselines in full (Table~\ref{tab:tasks}).

\subsection{Calibration and Statistical Inference}

The deep fallback approximates action-value uncertainty with a Bayesian last layer (BLL) and ensemble value heads. Primary paired contrasts use bootstrap 95\% confidence intervals (CIs) with Holm correction within each declared family. The primary configuration uses nominal query price $\kappa=0.005$ and acquisition threshold $c_t=0.05$; under $c_t=\lambda\kappa$, this operating point corresponds to $\lambda=10$. The separate price sweep uses $c_t=\kappa$, corresponding to $\lambda=1$. Appendix~\ref{app:impl-summary} distinguishes these parameterizations.

\section{Results}
\label{sec:results}


\subsection{Return and Query Volume}
\label{sec:return}

On GoToObj, the original controller improves return over never-query at both advisor scales while reducing calls by more than $97\%$ relative to always-query (Figure~\ref{fig:return-calls}). The empirical-radius gate improves on never-query by paired $0.029\pm0.016$ at 1.5B and $0.030\pm0.015$ at 7B. Holm-adjusted $p$-values are $0.005$ and $0.0006$ within the four-controller-versus-never family. The gate uses 75 and 82 advisor calls, compared with 3,408 and 3,403 for always-query, reductions of $97.8\%$ and $97.6\%$.

Always-query reduces return on both tasks, so comparison with that baseline alone overstates the gate's benefit. On GoToLocal, the gate uses 59--65 calls without a statistically resolved return difference from never-query. Pooling both tasks gives a nonsignificant $0.011\pm0.015$ effect. The supported return improvement is confined to GoToObj.

The uncalibrated gate also improves GoToObj over never-query, with unadjusted paired differences of $+0.024\pm0.018$ at 1.5B and $+0.026\pm0.016$ at 7B (Tables~\ref{tab:alg1-real},~\ref{tab:alg1-7b}), so the gain is not specific to the calibrated radius. On GoToObj, the empirical radius selects roughly half as many calls as the uncalibrated gate (75 versus 139 at 1.5B), without an observed return loss. This comparison concerns query volume; it does not identify whether the selected states are better. The formal radius is approximately $2.14$ on the controller's value scale and suppresses every call. Its abstention reflects an uninformative bound rather than a certified deployment benefit.

\begin{figure*}[!t]
\centering
\includegraphics[width=\textwidth]{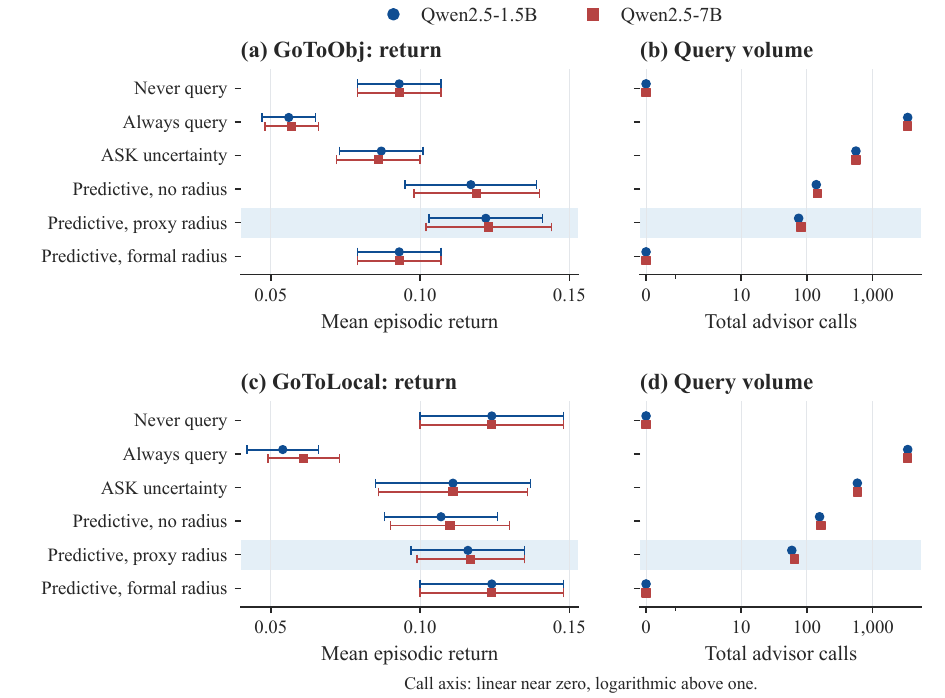}
\caption{Return and query volume of the original controller with real Qwen advisors (20 paired seeds, 60 episodes). Return bars show marginal 95\% CIs, not paired-effect intervals. Calls are summed over episodes and averaged over seeds on an axis that is linear near zero and logarithmic above one. Shading marks the empirical-radius gate ($\beta_{\rm conf}$).}
\label{fig:return-calls}
\end{figure*}

\subsection{Matched-Budget Query Placement}

The original controller does not establish
a matched-budget advantage over early scheduling.
In the primary comparison, its GoToObj return
exceeds random placement at both advisor scales,
although only the 1.5B contrast survives Holm
correction. At 1.5B, however, random placement
uses 68.5 rather than 74.8 calls on GoToObj
and 52.5 rather than 59.1 on GoToLocal,
corresponding to call deficits of 8.4\%
and 11.2\%. These differences confound
query placement with query volume.
We therefore use the separately executed
exact-call replication to assess placement
without this discrepancy.

A fresh 20-seed replication carries random-placement call deficits forward to match cumulative counts (Figure~\ref{fig:placement-calibration}(a)). It is separate from the primary evaluation in Figure~\ref{fig:return-calls}; only within-run paired differences are interpretable. Its absolute returns and calls cannot be combined with the primary matrix.

The fresh exact-call replication reports a
GoToObj 1.5B gate-minus-random difference
of $+0.019\pm0.018$. Its unadjusted interval
excludes zero, but a paired $t$ approximation
over 20 seeds gives $p\approx0.040$.
The corresponding approximate test does
not survive Holm correction over the eight
placement contrasts. Exact bootstrap-adjusted
$p$-values cannot be reconstructed from the
retained aggregate records. Consequently,
this replication does not establish a
multiplicity-adjusted placement advantage
(Appendix~\ref{app:complete-limit}). 

Figure~\ref{fig:placement-calibration}(a) shows no resolved advantage over the early schedule. Predictor and certificate-off ablations in the other matched matrix also show no resolved differences. These findings support limited early call volume without establishing better state selection.

%
%
%

\subsection{Query-Value Calibration}

The empirical radius covers most held-out opportunities marginally but undercovers those with positive estimated query value. The original controller's audits use 48 opportunities per seed in one run and 200 in a later chronological, groupwise run (Figure~\ref{fig:placement-calibration}(b); Tables~\ref{tab:beta-babyai} and~\ref{tab:mondrian}). Their coverage populations differ and must remain separate.

The later audit reports pooled coverage $0.81$ overall and $0.47$ on positive-value opportunities; Mondrian calibration raises the latter to $0.85$, below nominal $0.90$, which diagnoses undercoverage without establishing conditional or simultaneous validity (Figure~\ref{fig:placement-calibration}(b); Appendix~\ref{app:calibration-audits}).


\begin{figure*}[!t]
\centering
\includegraphics[width=\textwidth]{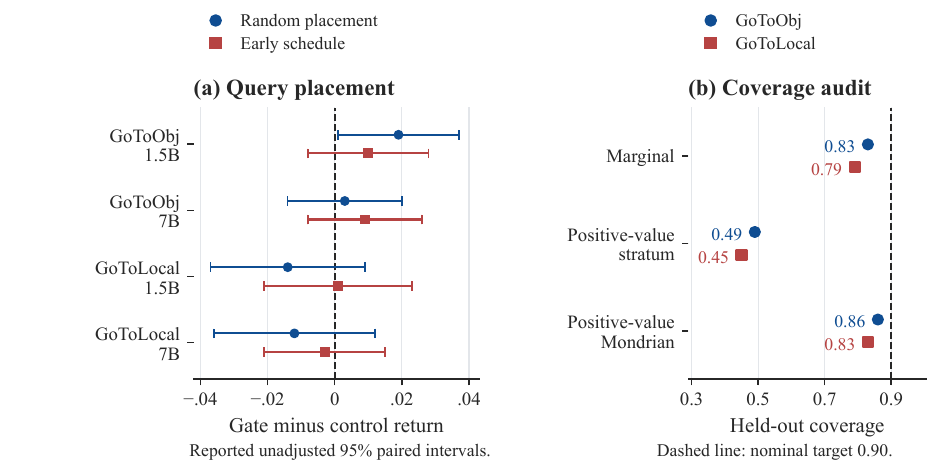}
\caption{Placement and calibration diagnostics for the original controller, from separate runs. (a) Gate-minus-control return in the fresh matched-budget run (20 paired seeds); bars are unadjusted 95\% paired intervals. (b) Held-out coverage in the later audit (200 calibration opportunities per seed); points are descriptive proportions without intervals, and the dashed line marks the 0.90 target.}
\label{fig:placement-calibration}
\end{figure*}

\subsection{Execution Screening}

Screening execution can reject both harmful and useful advice. In the 50\%-corruption test, the certificate reduces corrupted-action acceptance from $1.00$ to about $0.47$--$0.48$ at the same query opportunities, while rejecting roughly a third of correct advice. Disabling it increases measured return by $0.022\pm0.022$ and $0.026\pm0.015$ on the two tasks. This test establishes a screening effect, without a return gain or a formal safety guarantee for the deep learner. Appendix~\ref{sec:babyai-cert} gives the detailed screening analysis.

At fixed query opportunities, the 10-seed
GoToObj experiment evaluates execution
screening under controlled advisor corruption.
At a corruption rate of $0.50$, the
certificate accepts $0\%$, $47\%$, and
$72\%$ of corrupted advice at tolerances
$\epsilon=0.10$, $0.25$, and $0.50$,
respectively. Four reported unadjusted
paired return intervals exclude zero,
all favoring certificate removal:
$\epsilon=0.50$ at corruption $0.75$,
and $\epsilon=1.00$ at corruption
$0$, $0.50$, and $0.75$.
None survives Holm correction over the
20 nondegenerate comparisons.
The certificate therefore reduces
corrupted-advice execution without
establishing a corresponding return
improvement or a deployment-level
safety guarantee.

\subsection{Response to Query Price}

Raising the query price reduces the original controller's calls monotonically on all three tested tasks (Figure~\ref{fig:query-price-response}; details in Appendix~\ref{sec:price-app}). GoToObj calls fall from $4.16$ per episode at zero price to $0.56$ at $\kappa=0.01$, and calls reach zero on every task at $\kappa=0.05$, where return matches never-query. Each price changes deployment rather than repricing a saved trajectory. The sweep lacks paired net-utility intervals and does not isolate runtime; it establishes behavioral price response, while leaving net compute and utility gains unresolved.


\begin{figure*}[!t]
\centering
\includegraphics[width=\textwidth]{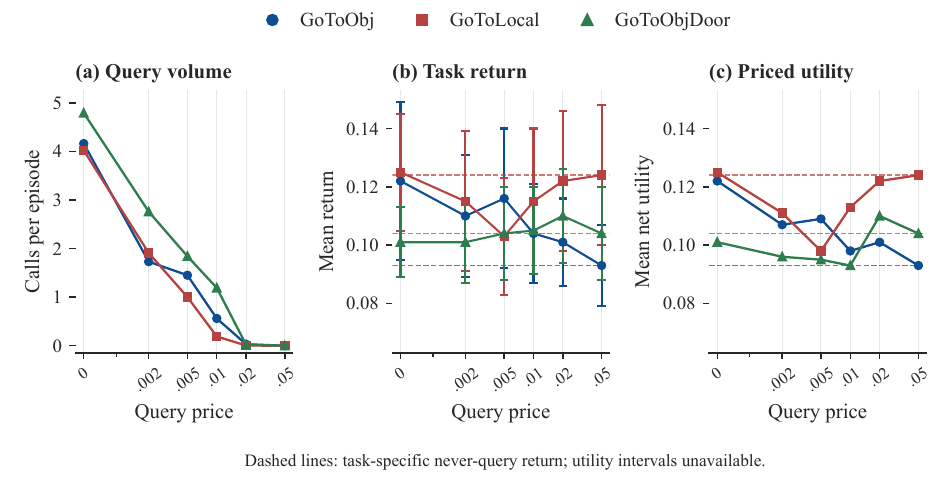}
\caption{Behavioral query-price response of the original controller with real Qwen2.5-1.5B advice (20 seeds, 60 episodes). Each price is a separate deployment run, with calibration shared per seed. (a) Advisor calls per episode. (b) Episodic return with 95\% CIs. (c) Net utility, return minus price times calls per episode, without paired intervals. Dashed lines mark never-query return; the price axis is linear near zero and logarithmic above 0.002.}
\label{fig:query-price-response}
\end{figure*}

\subsection{Headroom-Aware Variant Results}
\label{sec:headroom-extension}

The separately evaluated variant reports positive gate-minus-early-schedule differences on GoToObj and GoToObjDoor at both advisor scales and near-zero differences on GoToLocal and Open (Table~\ref{tab:headroom-placement}). These are descriptive contrasts: the retained aggregates do not permit reconstruction of paired CIs, verification of within-seed call matching, or multiplicity-adjusted inference. They do not replace the original matched-budget comparison.
With real Qwen2.5-7B advice, the extension also reports descriptive success fractions on BabyAI Unlock and PutNextLocal and on ALFWorld pick-and-place and look-at-object (Figure~\ref{fig:nlp-benchmark-breadth} in Appendix~\ref{app:headroom-ext}).
Held-out decision-stratum coverage ranges from $0.911$ to $0.920$ without establishing simultaneous coverage (Appendix~\ref{app:headroom-ext}).

\begin{table*}[t]
\centering\small
\setlength{\tabcolsep}{3pt}
\caption{Headroom-aware variant versus an early schedule across six task-advisor comparisons. Calls are reported allocations over 60 episodes; equality of realized calls within each seed is unverified. Differences and uncertainty half-widths remain as reported, without reconstructed paired CIs or Holm-adjusted tests. Point estimates are descriptive.}
\label{tab:headroom-placement}
\begin{tabular}{llrrrrr}
\toprule
Task & Advisor & Gate return $\uparrow$ & Early return $\uparrow$ & Calls & Paired $\Delta$ & No advice $\uparrow$\\
\midrule
GoToObj & 1.5B & $\mathbf{0.149}$ & $0.117$ & $86$ & $+0.032\pm0.012$ & 0.093\\
GoToObj & 7B & $\mathbf{0.151}$ & $0.122$ & $91$ & $+0.029\pm0.011$ & 0.093\\
GoToObjDoor & 1.5B & $\mathbf{0.142}$ & $0.115$ & $97$ & $+0.027\pm0.011$ & 0.104\\
GoToObjDoor & 7B & $\mathbf{0.147}$ & $0.119$ & $106$ & $+0.028\pm0.012$ & 0.104\\
GoToLocal & 1.5B & $\mathbf{0.125}$ & $0.124$ & $3$ & $+0.001\pm0.011$ & 0.124\\
Open & 7B & $\mathbf{0.014}$ & $0.013$ & $2$ & $+0.001\pm0.004$ & 0.012\\
\bottomrule
\end{tabular}
\end{table*}

\section{Discussion and Conclusion}

The original controller improves GoToObj return with few queries but shows no resolved matched-budget advantage over an early schedule. The separate headroom-aware extension reports positive schedule contrasts on GoToObj and GoToObjDoor and higher decision-relevant coverage. It does not validate the original empirical radius, which poorly covers consequential decisions, or its formal radius, which abstains throughout the original real-advisor evaluation.

The practical implication is to evaluate advice volume, query timing, and execution screening separately. Sparse querying can improve return without demonstrating better state selection, and reduced corrupted-advice execution can coexist with lower return. Controlled diagnostics and scope analyses further limit these findings (Appendices~\ref{app:controlled} and~\ref{sec:babyai-alg1}): global or uniform predictor replacements show no resolved return loss in BabyAI, and sparse querying on GoToObj does not establish task-level headroom detection across BabyAI.

Deployment guarantees require validated action bounds, decision-relevant acquisition calibration, and controlled policy selection across tasks and budgets. Computational savings require gate and calibration costs as well as saved calls (Appendix~\ref{app:gate-overhead}). These requirements remain unresolved. Appendix~\ref{sec:discussion} extends this discussion.

\section*{Limitations}
The original controller's primary return study uses two BabyAI levels, and its price and volume sweeps add GoToObjDoor, Open, and Pickup. Its primary return study lacks task-success rates and a ceiling on attainable return, and the low initial GoToObj return may leave headroom that disappears with a stronger fallback. The headroom-aware extension adds BabyAI and real-advisor ALFWorld evaluations, but complete per-seed analyses and corrected selection tests are unavailable. Earlier ALFWorld, Crafter, pixel-grid, and controlled-chain studies use simulated advisors and provide separate diagnostic evidence. Appendix~\ref{sec:limitations} gives extended limitations, and Appendix~\ref{app:controlled} reports the simulated-advisor diagnostics.

The theoretical guarantees assume action-value coverage and exchangeability that are not established for the deep controller. Acquisition is myopic under a hard budget. Per-decision gate timings and per-call advisor latency also leave total serving cost and calibration amortization unresolved.

\section*{Ethical Considerations}
Automated advice can contain incorrect, biased, or unsafe instructions. Passing a model-relative value bound does not certify physical safety, fairness, or the truth of the underlying language-model response. Deployed systems require task-level constraints, data minimization, rate limits, failure logging, and an independently validated fallback. The study uses public navigation benchmarks, not sensitive personal tasks. Appendix~\ref{sec:impact} discusses broader impact.

\section*{AI Use Statement}
Generative AI tools assist with language editing, literature discovery, mathematical exposition, and source-code preparation. The authors verify the proofs, citations, experimental records, and final claims and remain responsible for the manuscript.

\section*{Reproducibility Statement}
Appendices~\ref{app:full-controller} and~\ref{sec:theory} give controller definitions and complete proofs. Appendix~\ref{app:protocol}, including its reproducibility record (Appendix~\ref{sec:reproducibility}), documents the fallback learner, advisor and parser interfaces, hyperparameters, calibration splits, query budgets, and seed-control procedures. Appendices~\ref{sec:experiments} and~\ref{app:controlled} give the real-advisor protocols and controlled settings, and Appendix~\ref{app:headroom-ext} reports the separately evaluated headroom-aware variant. The evaluation distinguishes real language-model advice from simulated-advisor diagnostics and exact query-value labels from paired-rollout proxies. All results retain their stated numerical units, replication structure, and available uncertainty.

\bibliography{references}
\clearpage
\appendix
\makeatletter
\setlength{\@fptop}{0pt}
\setlength{\@fpsep}{12pt plus 2pt minus 2pt}
\setlength{\@fpbot}{0pt plus 1fil}
\setlength{\@dblfptop}{0pt}
\setlength{\@dblfpsep}{12pt plus 2pt minus 2pt}
\setlength{\@dblfpbot}{0pt plus 1fil}
\makeatother
\section{Extended Related Work}
\label{app:related}

Asking-for-help policies and teacher-advising methods select when external knowledge is useful \citep{torrey2013teaching,omidshafiei2019learningtoteach,dasilva2020uncertainty,liu2022askingforknowledge,hu2023when2ask}. Embodied agents also learn to request expert or simulated-human assistance \citep{nguyen2019helpanna,singh2022ask4help}, and ASK queries a small language model for actions only when Monte Carlo dropout uncertainty of a trained policy exceeds a threshold \citep{monteiro2026ask}. Our formulation averages over predicted responses through the deployed action rule, accounting for parsing failure, rejection, and cache reuse. The real-advisor experiments lack a matched Asking for Knowledge (AFK) implementation and do not establish higher return than learned advice policies.

The fallback draws on posterior-sampling and optimism-based RL \citep{osband2013more,russo2018tutorial,rlsvi,osband2016bootstrapdqn,osband2018randprior}. Comparing query value with price connects the acquisition rule to metareasoning and information acquisition \citep{ids,metareason,hay2012optimalstopping,sukhija2025maxinforl}. Language models also support exploration, reward design, planning, and action grounding \citep{ellm,reward,yao2023react,ahn2022saycan,wang2023voyager,huang2022zeroshot}; online RL can also ground an LLM used directly as the policy \citep{carta2023glam}. \citet{cao2024llmrlsurvey} survey LLM-enhanced RL.

Cost-aware prediction and routing compare model-call value with price \citep{frugal,chen2020frugalml,ong2024routellm,ding2024hybridllm,aggarwal2024automix}. Cascade analyses characterize when confidence-based deferral suffices and extend deferral to token-level uncertainty in generation \citep{jitkrittum2023cascade,gupta2024lmcascades}. In RL, calls also affect later states, learning, and advice value. Sequential conformal orchestration addresses related purchase decisions but lacks a matched comparison in this action-certificate and cached-policy setting.

Selective classification and learning to defer decide when a predictor abstains or passes a decision to an expert \citep{geifman2017selective,madras2018predict,mozannar2020consistent}. Distribution-free calibration gives finite-sample coverage or risk control for such decisions \citep{angelopoulos2023gentle,bates2021rcps,angelopoulos2024crc}, and KnowNo uses conformal prediction so that LLM planners ask for help with a statistical task-completion guarantee \citep{ren2023knowno}. These methods calibrate a prediction or plan step; our certificate instead concerns the value of purchasing advice whose effect propagates through an RL fallback.

Learned expert consultation trains the LLM itself as a policy with access to memory, tools, and experts \citep{feng2024agile}. Conformal Constrained Policy Optimization (CCPO) instead learns cost-aware sequences of model and tool calls under a reliability constraint \citep{si2026ccpo}. CCPO and learned consultation use different action and training interfaces from fixed primitive-action advice to a separate fallback learner. Their published scores therefore do not provide matched baselines here. Our comparisons establish economical advice volume on GoToObj, without establishing superiority to learned consultation.

\section{Notation}
\label{app:notation}

Table~\ref{tab:notation} summarizes the notation used in the main text and supporting analysis. Query acquisition uses the value of a declared continuation rule, while execution certification uses bounds on the optimal action value. Expectations and probabilities use the conditioning specified in each result.

\begin{table*}[!t]
\centering
\footnotesize
\setlength{\tabcolsep}{5pt}
\renewcommand{\arraystretch}{1.0}
\caption{Notation for advice, query acquisition, calibration, and analysis.}
\label{tab:notation}
\begin{tabular}{@{}>{\raggedright\arraybackslash}p{\dimexpr0.24\textwidth-5pt\relax}>{\raggedright\arraybackslash}p{\dimexpr0.76\textwidth-5pt\relax}@{}}
\toprule
Symbol & Meaning \\
\midrule
\multicolumn{2}{@{}l}{\textbf{Environment and Advisor}} \\
\(\calM=(\calS,\calA,P,r,H)\) & Episodic MDP: state space, action space, transition law, reward function, and horizon. \\
\(M,M^\star\) & A possible environment model and the latent environment generating the interaction. \\
\(t,h,T\) & Primitive interaction index, within-episode stage, and total number of primitive decisions. \\
\(\calH_t,s_t,z_t\) & History before decision \(t\), state or learned representation, and advisor context \(z_t=\phi(\calH_t)\). \\
\(Y_t,\calY,\calY_t\) & Advisor response, response space, and response set enumerated at opportunity \(t\), which may be the parsed-action set plus \(\bot\). \\
\(g(y,z_t),\bot\) & Parser mapping a response to an action; \(\bot\) denotes parsing failure or abstention. \\
\(p_\star,p_{\psi,t},\widehat p_t\) & Environment-conditional advisor response law, predictive response model, and response probabilities used by the estimator. \\
\(b_t,b_t^y\) & Environment belief before and after response \(y\); \(b_t^y=b_t\) for candidate-only advice. \\
\(\Prb,\E\) & Probability and expectation. \\
\midrule
\multicolumn{2}{@{}l}{\textbf{Actions and Continuation Values}} \\
\(a_t^{\rm base},a_t^{{\rm base},y}\) & Fallback proposal before a response and after any permitted response-conditioned belief update. \\
\(a_y,a_t^c,a_t^0\) & Parsed advice action, cached candidate, and deployed no-query action. \\
\(d_t(y)\) & Deployed post-response action, including fallback after parsing failure or certificate rejection. \\
\(\ell_t,\pi_t^{{\rm cont},y}\) & Declared continuation length and response-dependent continuation rule fixed before the fresh response. \\
\(Q_{M,t}^y,Q_{M,t}^0\) & Expected reward over the declared continuation under response \(y\) and the no-query branch. \\
\(\mu_{b,t}^y,Q_{k,t}^y,K\) & Belief-averaged continuation value, its \(k\)-th sampled value, and number of value samples. \\
\(V_t^y(a)\) & Abbreviation for \(\mu_{b_t^y,t}^y(s_t,a)\) in the post-response decision rule. \\
\(Q_{M,h}^\star,V_h^\star\) & Full-horizon optimal action and state values at stage \(h\); the model subscript is omitted when fixed. \\
\(L_t^b,U_t^b,\epsilon_t\) & Lower and upper bounds on the optimal action value at belief \(b\), and execution-gap tolerance. \\
\midrule
\multicolumn{2}{@{}l}{\textbf{Query Acquisition and Calibration}} \\
\(\Delta_t,\bar\Delta_t,\widehat\Delta_t\) & Latent operational query value, Bayes pre-posterior value, and practical estimate, each relative to the no-query branch and before price subtraction. \\
\(q_t,q_t^\star\) & Binary query decision and oracle myopic query decision. \\
\(B,B_t\) & Hard query budget and remaining budget at decision \(t\). \\
\(c_t,\lambda,\kappa_t\) & Priced query cost \(c_t=\lambda\kappa_t\), cost multiplier, and per-call cost; \(\kappa\) denotes a fixed price in cost sweeps. \\
\(\beta_t\) & Acquisition radius bounding query-value estimation error under the stated calibration assumptions. \\
\(i,n,J,R\) & Calibration opportunity index, number of opportunities, independent response blocks, and paired rollouts per block. \\
\(G^{\rm query}_{ijr},G^0_{ijr}\) & Paired continuation returns for the query and no-query branches. \\
\(\widetilde\Delta_i,\rho_i^{\rm MC}\) & Paired-rollout proxy for latent query value and its simultaneous MC error radius. \\
\(D_i\) & Prespecified deterministic bound on the absolute paired continuation-return difference. \\
\(x_i,s_\omega(x_i),r_i\) & Calibration features, separately fitted positive error scale, and scaled calibration score. \\
\(\alpha,q_{1-\alpha}\) & Target miscoverage and split-conformal quantile used to construct \(\beta_t\). \\
\(T_{\rm eval},\delta_{\rm conf}\) & Declared number of evaluation opportunities and error allocation for simultaneous conformal coverage. \\
\(\delta_{\rm MC},\delta_\beta\) & Failure levels for proxy-noise bounds and simultaneous query-value error coverage. \\
\midrule
\multicolumn{2}{@{}l}{\textbf{Theoretical and Experimental Analysis}} \\
\(\Gap_t(a)\) & Optimal Bellman gap \(V_h^\star(s_t)-Q_h^\star(s_t,a)\). \\
\(\cal C_T,\delta\) & Set of executed-advice steps and failure level for simultaneous action-value coverage. \\
\(\Reg(T),\mathfrak R_{\rm base}(T)\) & Controller regret and fallback-gap bound on histories induced by the wrapped controller. \\
\(W_T\) & Upper bound on the cumulative selected confidence width over fallback steps. \\
\(u_t(q)\) & Incremental myopic surrogate \(q(\Delta_t-c_t)\). \\
\(g_z,\eta_z,m_z,\epsilon_z,c_z\) & Fallback gap, certified-response probability, reuse count, certified gap tolerance, and query price in the repeated-context model. \\
\(\eta\) & Simulated-advisor reliability under the definition given for each controlled experiment. \\
\(N_{\rm calls},N_{\rm ep}\) & Number of advisor calls and episodes used to report call volume and cost-adjusted utility. \\
\bottomrule
\end{tabular}
\end{table*}

Some symbols are reused locally. In calibration, \(R\) counts paired rollouts; in utility expressions, it denotes return. The parser \(g\) is distinct from the context gap \(g_z\) and the gate wall time \(g\) used in the latency analysis. We define each meaning where it is used.

\section{Full Problem Formulation and Controller}\label{app:full-controller}

This section defines the response model, acquisition target, execution check, and cache behavior in full. Table~\ref{tab:notation} summarizes the symbols; Algorithm~\ref{alg:voi} combines the decisions into the deployed controller.
\subsection{Problem Formulation}
\label{sec:problem}

We consider an episodic MDP \(\calM=(\calS,\calA,P,r,H)\) with rewards in \([0,1]\), horizon \(H\),
and \(T\) primitive interaction steps. The unknown MDP is denoted by \(M\), and the agent's history before
time \(t\) is \(\mathcal H_t\). A fallback learner proposes an action \(a_t^{\rm base}\). In the tabular analysis
the learner maintains a belief \(b_t(M)=p(M\mid\mathcal H_t)\); in the deep implementation, posterior
samples from a BLL approximate the induced distribution over action values.

An external advisor receives a context \(z_t=\phi(\mathcal H_t)\) and returns a raw response
\(Y_t\in\calY\). A parser \(g(Y_t,z_t)\) maps the response to a candidate action in
\(\calA\cup\{\bot\}\), where \(\bot\) represents a parsing failure or abstention. Each call has priced cost
\(c_t=\lambda\kappa_t\), and the controller has a hard budget \(B\). The cache stores the response,
parsed action, context key, insertion time, confidence statistics, and advisor-model version.

The controller maintains a predictive response model
\begin{equation}
  p_{\psi,t}(y\mid z_t,\mathcal H_t).
  \label{eq:response-predictor}
\end{equation}
For an informative advisor, responses are observations about the latent task and a likelihood
\(p_\psi(y\mid M,z_t)\) induces the response-conditioned belief
\begin{equation}
 b_t^y(M)
 =
 \frac{p_\psi(y\mid M,z_t)b_t(M)}
 {\int p_\psi(y\mid M',z_t)b_t(M')\,dM'}.
 \label{eq:response-update}
\end{equation}
Candidate-only advice proposes an action without updating the environment belief: \(b_t^y=b_t\). Informative advice uses Bayesian value of information; candidate-only advice uses predictive value of advice. Each experiment identifies its interpretation.

Query acquisition and execution safety use different value objects. Before a query opportunity, the
controller fixes a continuation window \(\ell_t\le H\), a no-response continuation rule
\(\pi_t^{{\rm cont},0}\), and a prospective rule \(\pi_t^{{\rm cont},y}\) for every parsed response
\(y\). These rules include cache reuse and belief-dependent behavior; the map
\(y\mapsto\pi_t^{{\rm cont},y}\) is fixed before observing the fresh response. Write
\(C_t^y(s,a)\) for the condition that the initial state and action are \(s_t=s,a_t=a\)
and subsequent actions follow \(\pi_t^{{\rm cont},y}\). Then
\begin{equation}
\begin{aligned}
 Q_{M,t}^{y}(s,a)
 &=\E_M\!\left[\sum_{j=0}^{\ell_t-1}r_{t+j}
 \mid C_t^y(s,a)\right],\\
 \mu_{b,t}^{y}(s,a)
 &=\E_{M\sim b}[Q_{M,t}^{y}(s,a)].
\end{aligned}
 \label{eq:declared-continuation}
\end{equation}
with superscript \(0\) for the no-query branch. This is the operational, explicitly myopic value used for
acquisition. It equals an optimal continuation value only in experiments that deliberately choose an
optimal continuation rule.

Execution safety instead targets the full-horizon optimal value \(Q_{M,h}^{\star}\). The controller
maintains action-specific confidence bounds \(L_t^b(s,a)\) and \(U_t^b(s,a)\) for that target. Their
empirical deep-RL estimates are not treated as formal confidence sequences unless explicitly calibrated
and evaluated as such. Keeping Eq.~\eqref{eq:declared-continuation} separate from these bounds prevents a
paired rollout estimate of a deployed policy from being mislabeled as an estimate of \(Q^\star\).

\subsection{Certified Predictive Advice Gating}
\label{sec:method}

\subsubsection{Certified Execution}

At belief \(b\), an advice action \(a_c\) is \(\epsilon_t\)-certified at state \(s_t\) when
\begin{equation}
 L_t^b(s_t,a_c)
 \ge
 \max_{a\in\calA}U_t^b(s_t,a)-\epsilon_t.
 \label{eq:safe}
\end{equation}
The advised action's confidence width must be small relative to the tolerance, making certification conservative early in learning. We report false-accept, false-reject, and overall acceptance rates.

Given a response \(y\), write \(a_y=g(y,z_t)\) and abbreviate
\(V_t^y(a)=\mu_{b_t^y,t}^{y}(s_t,a)\). The post-response decision is
\begin{equation}
 d_t(y)=
 \begin{cases}
 a_y, & \substack{a_y\ne\bot,\ \text{Eq.~\eqref{eq:safe} holds},\\
 V_t^y(a_y)\ge V_t^y(a_t^{{\rm base},y})},\\[2mm]
 a_t^{{\rm base},y}, & \text{otherwise}.
 \end{cases}
 \label{eq:post-response-action}
\end{equation}
where \(a_t^{{\rm base},y}\) is the fallback proposal after any permitted response-conditioned belief
update. For candidate-only advice, it is the unconditioned fallback action unless the implementation explicitly
resamples. Thus every possible response is evaluated through the same decision rule that governs execution
after the call.

\subsubsection{No-Query Action and Cache Reuse}

Let \(a_t^c\) be the cached candidate for the current context, if one exists. The no-query action is
\begin{equation}
 a_t^0=
 \begin{cases}
 a_t^c, & \substack{a_t^c\text{ is certified, and}\\
 \mu_{b_t,t}^{0}(s_t,a_t^c)\ge
 \mu_{b_t,t}^{0}(s_t,a_t^{\rm base})},\\[1mm]
 a_t^{\rm base}, & \text{otherwise}.
 \end{cases}
 \label{eq:no-query-action}
\end{equation}
This definition prevents a certified but lower-valued cached action from overriding the fallback merely
because it lies within a permissive tolerance.

\subsubsection{Pre-Query Value}

For the latent environment \(M^\star\), the operational response-contingent value of a query is
\begin{equation}
\begin{aligned}
 \Delta_t
 &=\E_{Y\sim p_\star(\cdot\mid M^\star,z_t,\mathcal H_t)}
 \left[Q_{M^\star,t}^{Y}(s_t,d_t(Y))\right]\\
 &\quad-Q_{M^\star,t}^{0}(s_t,a_t^0).
\end{aligned}
 \label{eq:true-query-value}
\end{equation}
Here \(p_\star\) is the advisor's environment-conditional response law, including the declared prompt and
decoding randomization. This counterfactual quantity is not observed online. It values the decision and continuation specified for
each possible response, not the raw advised action. The corresponding Bayes pre-posterior value is
\begin{equation}
\begin{aligned}
 \bar\Delta_t
 &=\E_{Y\sim p_{\psi,t}(\cdot\mid z_t,\mathcal H_t)}
 \left[\mu_{b_t^Y,t}^{Y}(s_t,d_t(Y))\right]\\
 &\quad-\mu_{b_t,t}^{0}(s_t,a_t^0).
\end{aligned}
 \label{eq:bayes-query-value}
\end{equation}
Under the informative likelihood model, Eq.~\eqref{eq:bayes-query-value} integrates the joint law of
\((M,Y)\) through \(b_t^Y\). In the candidate-only case it is the predictive value of advice and
\(b_t^Y=b_t\). The estimator below approximates \(\bar\Delta_t\); the calibrated radius covers its total
error relative to the latent operational target \(\Delta_t\).

With response samples or an enumerated parsed-action distribution and \(K\) posterior value samples, the
practical estimator is
\begin{equation}
\begin{aligned}
 \widehat\Delta_t
 &=\sum_{y\in\calY_t}\widehat p_t(y\mid z_t)
 \left[\frac{1}{K}\sum_{k=1}^{K}Q_{k,t}^y(s_t,d_t(y))\right]\\
 &\quad-\frac{1}{K}\sum_{k=1}^{K}Q_{k,t}^{0}(s_t,a_t^0).
\end{aligned}
 \label{eq:query-estimator}
\end{equation}
where the \(Q_{k,t}^y\) are samples of the declared continuation value in
Eq.~\eqref{eq:declared-continuation}, and \(\calY_t\) can be the finite parsed-action set plus \(\bot\).
Reusing unconditional samples is valid only in the candidate-only case. An informative advisor requires
response-conditioned model samples or importance weights representing Eq.~\eqref{eq:response-update}.

\subsubsection{Calibrating the Query-Value Radius}
\label{sec:beta-calibration}

We estimate the query-value radius end to end from calibration data. A randomized acquisition
policy collects a disjoint calibration set of query opportunities. At opportunity \(i\), the fallback learner,
response predictor, prospective continuation map, and cache are frozen. We draw \(J\) independent advisor-response blocks
and, for every response, clone the same latent simulator state and run \(R\) paired continuations with
common random numbers: one executes \(d_i(Y)\) followed by \(\pi_i^{{\rm cont},Y}\), and the other
executes \(a_i^0\) followed by \(\pi_i^{{\rm cont},0}\). Thus the proxy uses exactly the continuation
window and branch rules in Eqs.~\eqref{eq:declared-continuation}--\eqref{eq:true-query-value}. It is not a
proxy for \(Q^\star\). This gives
\begin{equation}
 \widetilde\Delta_i
 =\frac{1}{JR}\sum_{j=1}^{J}\sum_{r=1}^{R}
 \left(G^{\rm query}_{ijr}-G^{0}_{ijr}\right),
 \label{eq:delta-proxy}
\end{equation}
where each return difference has absolute value at most a prespecified deterministic bound $D_i$. For per-step rewards in $[0,1]$, a declared continuation gives $D_i=\ell_i\le H$; using $D_i=H$ is also valid. A simultaneous MC radius for $n$ calibration opportunities is
\begin{equation}
 \rho_i^{\rm MC}
 =D_i\sqrt{\frac{2\log(2n/\delta_{\rm MC})}{J}},
 \label{eq:mc-radius}
\end{equation}
where each of the \(J\) independent outer blocks averages \(R\) paired continuations. Continuations sharing a response are not independent response samples. For correlated calls, one declared cluster forms an outer block, and \(J\) counts clusters. The radius has the form of Hoeffding's inequality \citep{hoeffding1963} for the outer-block means with a union bound over opportunities. A response-block empirical-Bernstein radius is an analytic alternative without faithful-controller measurements; time-uniform confidence sequences \citep{howard2021timeuniform,waudbysmith2024betting} are a further alternative when audits stop adaptively.

A scale model \(s_\omega(x)>0\), fitted on data disjoint from the calibration set, predicts how query-value
error varies with observation, predictor entropy, value width, and parser-failure probability. Calibration
scores are
\begin{equation}
 r_i=\frac{|\widehat\Delta_i-\widetilde\Delta_i|+\rho_i^{\rm MC}}
 {s_\omega(x_i)}.
 \label{eq:beta-score}
\end{equation}
Let \(q_{1-\alpha}\) be the split-conformal quantile with rank
\(\lceil(n+1)(1-\alpha)\rceil\), using an infinite radius if the requested rank exceeds \(n\). The deployed
radius is
\begin{equation}
 \beta_t=q_{1-\alpha}s_\omega(x_t).
 \label{eq:beta-radius}
\end{equation}
Under exchangeability between randomized calibration opportunities and the declared evaluation
opportunities, this construction gives marginal coverage for the total error between \(\widehat\Delta_t\)
and the latent \(\Delta_t\), including response-model, value-model, and posterior-predictive error, while
Eq.~\eqref{eq:mc-radius} accounts for proxy noise. For a declared finite collection of at most
\(T_{\rm eval}\) evaluated opportunities, setting
\(\alpha=\delta_{\rm conf}/T_{\rm eval}\) and applying a union bound gives the simultaneous event required
by Assumption~\ref{ass:value-error}, with
\(\delta_\beta\le\delta_{\rm conf}+\delta_{\rm MC}\). The resulting radius can be conservative or
infinite when calibration data are insufficient. We therefore report marginal coverage, simultaneous
trajectory coverage, radius width, and the fraction of opportunities on which the conservative gate
abstains.

This procedure is feasible in BabyAI, TextWorld/ALFWorld, Crafter, and the controlled tasks because the
advisor is real but the environment is clonable. In a non-clonable deployment, \(\Delta_t\) cannot be
counterfactually audited this way. There the radius is only a predictive heuristic, and
Theorem~\ref{thm:allocation} is not claimed to hold empirically without an alternative identification
argument.

The conservative gate is
\begin{equation}
 q_t=1
 \quad\Longleftrightarrow\quad
 \widehat\Delta_t-\beta_t\ge c_t,
 \quad B_t>0.
 \label{eq:query-rule}
\end{equation}
The nonconservative ablation sets \(\beta_t=0\). A target budget can be enforced with an online dual
update for \(\lambda\), or \(\lambda\) can be calibrated on a held-out training prefix. Test returns are not
used to select it.

\subsubsection{Illustrative Decision}\label{app:example}
Consider an agent with no-query value $0.41$. Its predictor assigns probability $0.30$ to a response that passes execution checks and yields $0.54$; other responses fail or are rejected, yielding $0.41$. The predicted gain is $0.30(0.54-0.41)=0.039$. With illustrative radius $0.010$ and price $0.020$, the gate asks because $0.039-0.010>0.020$. This deterministic example, not an experimental result, shows why purchase value must account for rejection.

\subsubsection{Cache Invalidation}

Cache validity is checked separately from query acquisition. An entry is disabled when its parsed action
fails Eq.~\eqref{eq:safe}, when the current context is outside the calibrated support of the response
predictor, or when a shift detector flags the context. A time-to-live policy is included as a baseline.
The support signal can be supplied by a latent density certificate, but value-space and
density-space triggers are reported separately because one does not imply the other.
Cache invalidation is not claimed to improve performance unless it beats both no-cache and time-to-live
under a protocol that reserves enough post-shift query budget for the policies to differ.


\section{Theory and Proofs}
\label{sec:theory}

This appendix proves the results of Sections~\ref{sec:gating} and~\ref{sec:guarantees}, then collects supporting results and intuition, local-range confidence bounds, and a variance-adaptive proxy-error option.

Execution safety, myopic allocation, and fallback regret require separate assumptions. A certificate does not establish purchase value, and standalone fallback regret bounds do not automatically survive advice-induced trajectory changes.

\subsection{Guarantees and Proofs}

For a visited state-stage pair, define the optimal Bellman gap
\begin{equation}
 \Gap_t(a)=V_h^\star(s_t)-Q_h^\star(s_t,a).
\end{equation}

\begin{assumption}[Uniform action-value coverage]
\label{ass:coverage}
With probability at least \(1-\delta\), simultaneously for every visited \(t,h,s,a\),
\begin{equation}
 L_t(s,a)\le Q_h^\star(s,a)\le U_t(s,a).
 \label{eq:coverage}
\end{equation}
\end{assumption}

\begin{lemma}[Certified advice gap]
\label{lem:certified}
On the event in Assumption~\ref{ass:coverage}, every advice action executed under
Eq.~\eqref{eq:safe} satisfies \(\Gap_t(a_t)\le\epsilon_t\).
\end{lemma}

\begin{proof}
Let \(a_t^\star\in\arg\max_a Q_h^\star(s_t,a)\). Coverage and the certificate give
\[
\begin{aligned}
Q_h^\star(s_t,a_t^\star)
&\le U_t(s_t,a_t^\star)
\le \max_a U_t(s_t,a)\\
&\le L_t(s_t,a_t)+\epsilon_t
\le Q_h^\star(s_t,a_t)+\epsilon_t.
\end{aligned}
\]
Rearranging proves the claim.
\end{proof}

\begin{assumption}[Intervention-stable fallback]
\label{ass:stable-base}
Let \(\mathcal C_T\) be the set of times at which certified advice is executed. On the histories induced by
the wrapped controller and on the coverage event, the fallback proposals satisfy
\begin{equation}
 \E\left[\sum_{t\notin\mathcal C_T}\Gap_t(a_t^{\rm base})\right]
 \le \mathfrak R_{\rm base}(T).
 \label{eq:stable-base}
\end{equation}
\end{assumption}

Assumption~\ref{ass:stable-base} concerns intervention-induced histories, so a standalone regret bound is insufficient. It requires a base analysis valid under the resulting adaptive data collection or separate verification for the posterior-sampling construction.

The condition can be verified for a concrete optimistic fallback. The following lemma relies only on action-specific bounds already computed by the certificate and remains valid when
advice changes the visited states.

\begin{lemma}[Adaptive-history stability of an optimistic fallback]
\label{lem:optimistic-stability}
Suppose the base proposal satisfies
\(a_t^{\rm base}\in\arg\max_a U_t(s_t,a)\), Assumption~\ref{ass:coverage} holds, and along every
adaptively generated trajectory
\begin{equation}
 \sum_{t\notin\mathcal C_T}
 \left(U_t(s_t,a_t^{\rm base})-L_t(s_t,a_t^{\rm base})\right)
 \le W_T.
 \label{eq:width-sum}
\end{equation}
Then Assumption~\ref{ass:stable-base} holds with \(\mathfrak R_{\rm base}(T)=W_T\). In particular, if
the width of state-stage-action triple \(x=(h,s,a)\) on its \(n\)-th visit is at most
\(CH\sqrt{\iota/\max\{1,n\}}\), then
\begin{equation}
 W_T\le 2CH\sqrt{HSA T\iota}.
 \label{eq:optimistic-width-bound}
\end{equation}
\end{lemma}

\begin{proof}
Let \(a_t^\star\) be optimal. On the coverage event and a fallback step,
\[
\begin{aligned}
Q_h^\star(s_t,a_t^\star)
&\le U_t(s_t,a_t^\star)
\le U_t(s_t,a_t^{\rm base}),\\
Q_h^\star(s_t,a_t^{\rm base})
&\ge L_t(s_t,a_t^{\rm base}).
\end{aligned}
\]
Thus \(\Gap_t(a_t^{\rm base})\) is at most the selected action's confidence width. Summing gives
Eq.~\eqref{eq:width-sum}. For the visit-count specialization, use
\(\sum_{n=1}^{N_x}n^{-1/2}\le2\sqrt{N_x}\), then Cauchy--Schwarz:
\(\sum_{x}\sqrt{N_x}\le\sqrt{HSA\sum_xN_x}\le\sqrt{HSA T}\).
\end{proof}

Lemma~\ref{lem:optimistic-stability} is a complete verification for the optimistic tabular fallback used in
the controlled study. Posterior sampling remains the primary randomized fallback, for which
Corollary~\ref{cor:psrl} is conditional; the chain experiment additionally compares its fallback-gap growth
under standalone and wrapped histories.

\begin{theorem}[Regret of the certified wrapper]
\label{thm:reg}
Under Assumptions~\ref{ass:coverage} and~\ref{ass:stable-base}, Algorithm~\ref{alg:voi} satisfies
\begin{equation}
 \E[\Reg(T)]
 \le
 \mathfrak R_{\rm base}(T)
 +\E\left[\sum_{t\in\mathcal C_T}\epsilon_t\right]
 +\delta TH.
 \label{eq:regret-bound}
\end{equation}
\end{theorem}

\begin{proof}
The finite-horizon performance-difference identity writes expected regret as the expected sum of optimal
Bellman gaps along the controller's own trajectory. On the coverage event, Lemma~\ref{lem:certified}
bounds the gaps on \(\mathcal C_T\), and Assumption~\ref{ass:stable-base} bounds the remaining gaps.
On the failure event, each primitive-step continuation gap is at most \(H\). Taking expectations gives
Eq.~\eqref{eq:regret-bound}.
\end{proof}

\begin{corollary}[Conditional PSRL specialization]
\label{cor:psrl}
If a particular PSRL construction is proved to satisfy Assumption~\ref{ass:stable-base} with
\(\mathfrak R_{\rm base}(T)=\mathfrak R_{\PSRL}(T)\), then the certified wrapper inherits that rate plus
the two additive terms in Eq.~\eqref{eq:regret-bound}. The ordinary standalone PSRL bound alone is not
sufficient for this corollary.
\end{corollary}

\begin{assumption}[Calibrated query-value error]
\label{ass:value-error}
With probability at least \(1-\delta_\beta\), simultaneously at every declared evaluation opportunity
before budget exhaustion,
\(|\widehat\Delta_t-\Delta_t|\le\beta_t\).
\end{assumption}

For one query opportunity, define the incremental myopic surrogate
\begin{equation}
 u_t(q)=q(\Delta_t-c_t),\qquad q\in\{0,1\}.
 \label{eq:myopic-surrogate}
\end{equation}
The oracle for this surrogate is \(q_t^\star=\mathbf{1}\{\Delta_t\ge c_t\}\), where
\(\mathbf{1}\) is the binary indicator.

\begin{theorem}[Surrogate loss of conservative query allocation]
\label{thm:allocation}
Let the oracle myopic gate query when \(\Delta_t\ge c_t\). On the event in
Assumption~\ref{ass:value-error}, the conservative rule in Eq.~\eqref{eq:query-rule} never issues a
negative-net-value query. Relative to the oracle myopic gate, its
cumulative loss in the surrogate of Eq.~\eqref{eq:myopic-surrogate}, before hard-budget exhaustion, is at most
\begin{equation}
 2\sum_{t=1}^{T}\beta_t.
 \label{eq:allocation-regret}
\end{equation}
\end{theorem}

\begin{proof}
If the conservative gate queries, then
\(\Delta_t\ge\widehat\Delta_t-\beta_t\ge c_t\), so the true net value is nonnegative. The only disagreement
with the oracle is a missed query. In that case
\(\widehat\Delta_t-\beta_t<c_t\), while
\(\Delta_t\le\widehat\Delta_t+\beta_t<c_t+2\beta_t\). The missed oracle gain is therefore below
\(2\beta_t\). Summing proves the result. A binding hard budget adds the opportunity cost of calls the
oracle reserves for later; that nonmyopic term is evaluated empirically and is not hidden in the theorem.
\end{proof}

Theorem~\ref{thm:allocation} does not bound end-to-end return regret. The myopic oracle itself can be
arbitrarily suboptimal when an early call changes later learning or consumes budget for a call with greater
value later. In the separable repeated-context model below, local improvement translates into a
decision-regret reduction. Outside that model, we report both surrogate allocation loss and realized return.

\begin{proposition}[Positive value under repeatable covered contexts]
\label{prop:positive}
Consider a decision context \(z\) whose action does not change the distribution of future occurrences of
that same context. Suppose the fallback action has expected gap at least \(g_z\), a query returns an action
with certified gap at most \(\epsilon_z\) with probability \(\eta_z\), failed advice falls back without
additional decision regret, and successful advice is reused for \(m_z\) occurrences. Then one query reduces
expected decision regret, net of price \(c_z\), by at least
\begin{equation}
 \eta_z m_z(g_z-\epsilon_z)-c_z.
 \label{eq:positive-value}
\end{equation}
It is beneficial whenever the right-hand side is positive.
\end{proposition}

\begin{proof}
On a successful response, each of the \(m_z\) covered occurrences replaces an action with expected gap at
least \(g_z\) by one with gap at most \(\epsilon_z\). Failure uses the base action. Taking expectation over
success and subtracting the query price gives Eq.~\eqref{eq:positive-value}.
\end{proof}

Proposition~\ref{prop:positive} applies to the controlled chain's repeated-context regime. It does not assume arbitrary interventions preserve state distributions; broader deep experiments measure the effect separately.

\begin{assumption}[Advice-potential contraction]
\label{ass:potential}
There is a nonnegative adapted potential \(\Phi_t\) measuring remaining realizable advice value. Whenever
the controller queries, its conditional expected decrease is at least \(d>0\); when it does not query, the
potential does not increase in expectation.
\end{assumption}

\begin{proposition}[Query accounting]
\label{prop:queries}
Under Assumption~\ref{ass:potential}, the expected number of queries before the potential is exhausted is
at most
\begin{equation}
 \E[N_{\rm query}]\le \frac{\Phi_0}{d}.
 \label{eq:query-bound}
\end{equation}
If a correct response decreases potential by \(d_+\), an incorrect response followed by fallback
contraction decreases it by \(d_-\), and responses are conditionally correct with probability \(\eta\), then
\(d=\eta d_+ +(1-\eta)d_-\). The bound is non-increasing in reliability when \(d_+\ge d_-\).
\end{proposition}

\begin{proof}
Summing the conditional expected potential decreases over query times and using nonnegativity gives
\(d\E[N_{\rm query}]\le\Phi_0-\E[\Phi_{\rm final}]\le\Phi_0\). Substituting the reliability-weighted
decrease yields the second statement.
\end{proof}

%


\subsection{Candidate-Only Value Limit}
\label{sec:candidate-limit}
\label{app:theory}

\begin{proposition}[Candidate-only advice cannot improve a fully known greedy decision]
\label{prop:candidate-only}
Suppose $b_t$ and the continuation value function do not change upon a response,
$Q^y_{M,t}=Q^0_{M,t}$ for every $y$, and $a_t^0$ maximizes the same posterior
mean continuation value over all available actions. If a response only proposes
an element of that same finite action set, its Bayes pre-query value is
nonpositive: $\bar\Delta_t\le0$.
\end{proposition}
\begin{proof}
For any parsed response $y$, the executed action $d_t(y)$ belongs to the same
action set, so by the maximizing property,
$\mu_{b_t,t}^0(s_t,d_t(y))\le\mu_{b_t,t}^0(s_t,a_t^0)$.
Taking the predicted-response expectation and subtracting the latter proves
the claim. The inequality is unaffected by parse failure or certificate
rejection because both replace a candidate with an available fallback action.
\end{proof}

Candidate-only Qwen gains must arise from a suboptimal or constrained fallback, an altered continuation, or an estimate differing from exact Bayes value. A response assumed conditionally independent of the latent environment cannot supply information gain.

\subsection{Fixed-Opportunity Budget Frontier}

\begin{proposition}[Offline frontier for fixed opportunities]
\label{prop:frontier}
For a fixed response model, fixed value estimates, and fixed set of query opportunities, selecting calls in
decreasing nonnegative estimated net value produces non-decreasing estimated utility as budget grows. The
frontier is concave when ordered marginal values are non-increasing.
\end{proposition}

\begin{proof}
Increasing the budget by one permits one additional nonnegative marginal value. The cumulative sum is
non-decreasing. Its successive increments are non-increasing when the ordered marginals are non-increasing,
which is the definition of discrete concavity.
\end{proof}

This proposition is an offline statement about an unchanged opportunity set. It does not imply that online
training return is monotone in budget, because an additional query changes later states, learning updates,
and query opportunities.

\subsection{Limit Behavior}

Table~\ref{tab:limits} summarizes the limiting cases. When \(c_t\) approaches zero, the acquisition rule queries whenever the lower confidence value is positive;
the returned action still requires certification. When cost diverges, no new response is purchased and the
controller reduces to the fallback learner plus any already certified cache. A reliable cheap advisor can
resolve repeated contexts quickly. An adversarial or repetitive advisor can waste calls unless its predicted
value contracts; the certificate limits execution but does not refund acquisition cost.

\begin{table}[t]
\centering
\small
\setlength{\tabcolsep}{3pt}
\renewcommand{\arraystretch}{1.15}
\caption{Acquisition and execution are distinct at important limits.}
\begin{tabular}{@{}>{\raggedright\arraybackslash}p{\dimexpr0.26\columnwidth-4pt\relax}
>{\raggedright\arraybackslash}p{\dimexpr0.40\columnwidth-4pt\relax}
>{\raggedright\arraybackslash}p{\dimexpr0.34\columnwidth-4pt\relax}@{}}
\toprule
Regime & Acquisition & Execution\\
\midrule
\(c_t\to0\) & query positive lower-bound value & certified action only\\
\(c_t\to\infty\) & no fresh query & fallback or certified cache\\
Reliable advisor & sparse discovery and refresh & often accepted\\
Correlated wrong advisor & predictor should suppress repeats & rejected if uncertified\\
\bottomrule
\end{tabular}
\label{tab:limits}
\end{table}

\subsection{Decision Value Versus Uncertainty}

Total predictive uncertainty combines reducible epistemic uncertainty with irreducible aleatoric variation.
A noisy reward state can have high total variance even when every plausible model selects the same action.
Querying there cannot improve the decision. Novelty has a similar limitation: an unseen state can have an
obvious dominant action, while a familiar fork can remain decision-sensitive.

Response-contingent value asks whether a possible answer changes the continuation decision enough to repay
its cost. The certificate asks a different question: whether one observed answer is sufficiently close to the
best action under valid bounds. Treating one scalar uncertainty measure as both acquisition value and an
execution guarantee confuses these roles.

\subsection{Value of Information Versus Value of Advice}

In the informative case, the response likelihood in Eq.~\eqref{eq:response-update} correlates the observed
answer with latent task values. The posterior after the answer can therefore differ from the posterior before
it. In the candidate-only case, the answer merely proposes an action that the existing critic can already
evaluate. Calling the latter information value obscures the missing correlation. Equation
\eqref{eq:true-query-value} covers both cases while naming them separately.

This distinction is especially visible in a small discrete action space. If the value model already evaluates
every action and the advisor response is independent of the true values conditional on history, observing a
candidate reveals no new fact about \(Q^\star\). It can still reduce search or computation in a large
structured action space, but that benefit must be modeled as candidate-generation value rather than a
Bayesian update.

\subsection{Local-Range Confidence Bounds}\label{app:bounded-and-audited}

This subsection tightens the MC radius of Eq.~\eqref{main:mc} for short continuations, bounds query value by decision-change probability, gives fixed-stratum conformal coverage, and records a variance-adaptive proxy-error option.

The independent sampling unit is an advisor-response block. For a short declared continuation, its reward range can tighten the MC radius without changing assumptions. This analytic refinement alone is not an empirical result.

\begin{lemma}[Finite-window paired proxy concentration]
\label{lem:local-mc}
For a fixed query opportunity $i$, freeze the fallback, response model,
post-response decision rule, and continuation policies. Suppose $J_i$
independent advisor-response blocks produce block-averaged paired-return
differences $X_{ij}$, with $X_{ij}\in[-\ell_i,\ell_i]$ almost surely and
$\mathbb E[X_{ij}]=\Delta_i$. If at most $n$ such opportunities are audited,
then with probability at least $1-\delta_{\rm MC}$ they simultaneously satisfy
\begin{equation}
\label{eq:local-mc}
 |\widetilde\Delta_i-\Delta_i|\le
 \rho_i^{\rm local}:=\ell_i
 \sqrt{\frac{2\log(2n/\delta_{\rm MC})}{J_i}}.
\end{equation}
Cloning additional paired transition continuations at the same response does
not increase $J_i$.
\end{lemma}
\begin{proof}
The difference of two $[0,\ell_i]$ continuation returns is bounded in
$[-\ell_i,\ell_i]$. Hoeffding's inequality applied to the $J_i$
independent block averages gives
$\Pr(|\overline X_i-\Delta_i|>x)\le
2\exp[-J_i x^2/(2\ell_i^2)]$. Substitute
$x=\rho_i^{\rm local}$ and union-bound the $n$ opportunities.
\end{proof}

Setting $D_i=H$ in Eq.~\eqref{main:mc} recovers the original valid bound; $D_i=\ell_i$ gives the local-range specialization because $\ell_i\le H$. The record lacks per-response blocks and a rerun under the local-range gate, so Lemma~\ref{lem:local-mc} does not establish BabyAI nonvacuity. An empirical-Bernstein alternative retains a finite-range term that can dominate with few independent blocks, even at near-zero variance. Cloned continuations cannot inflate the independent sample size.

\begin{proposition}[Decision-change envelope]
\label{prop:decision-change}
At a fixed opportunity let $E_y$ denote the event that the query and
no-query branches execute the same action and the same declared
continuation policy, so their counterfactual returns agree. If their
returns are in $[0,\ell_t]$, then
\begin{equation}
\label{eq:change-envelope}
 |\Delta_t|\le\ell_t\Pr_{Y\sim p_\star}(E_Y^c\mid\cal H_t).
\end{equation}
\end{proposition}
\begin{proof}
Condition on whether the two prospective branches are identical.
Their return difference is zero on $E_Y$, and its absolute value is at
most $\ell_t$ otherwise. Take the conditional expectation and apply the
triangle inequality.
\end{proof}

Equation~\eqref{eq:change-envelope} permits pre-query screening if the predictor supplies a valid upper bound on action- or continuation-changing response mass. The BabyAI predictor lacks this bound, so the shortcut remains unevaluated. Certification must account for prediction error, including correlated wrong responses.

\begin{lemma}[Fixed-stratum conformal error coverage]
\label{lem:fixed-stratum}
Let a stratum $G(x)\in\{1,\ldots,m\}$ be defined using features observable
before a fresh advisor response and frozen before the calibration data are
used. Suppose the $n_g$ calibration residuals
$R_i=|\widehat\Delta_i-\Delta_i|/s(x_i)$ in stratum $g$, and a
future residual in that stratum, are exchangeable. If $s(x)>0$ is fitted
on disjoint data and $q_g$ is their rank
$\lceil(n_g+1)(1-\alpha_g)\rceil$ split-conformal quantile, with
$q_g=\infty$ whenever that rank exceeds $n_g$, then
\begin{equation}
\label{eq:fixed-stratum-conformal}
 \Pr\bigl(|\widehat\Delta-\Delta|\le q_g s(x)
 \mid G(x)=g\bigr)\ge1-\alpha_g.
\end{equation}
\end{lemma}
\begin{proof}
Conditional exchangeability within the fixed stratum gives a
super-uniform rank for the held-out residual among the $n_g+1$ scores.
The split-conformal order statistic consequently includes that residual
with probability at least $1-\alpha_g$.
\end{proof}

The lemma does not grant validity after drifting to a new state distribution,
after selecting states using a score not represented in the calibration
stratum, or after substituting an inaccurate paired-rollout label for
$\Delta_i$. To include proxy noise, an independently justified error
radius must be added with its own failure budget. To obtain simultaneous
coverage across a prespecified finite evaluation horizon, allocate a
suitable per-opportunity failure budget and union-bound it; this may force
$q_g=\infty$ with the available calibration sample size.

\paragraph{Variance-adaptive proxy-error option.}
For independent response-block means $X_{ij}\in[-H,H]$ and $J\ge2$, write
$\widehat v_i=(J-1)^{-1}\sum_j (X_{ij}-\bar X_i)^2$. A standard empirical
Bernstein form supplies a candidate uniform correction of order
\[
\rho_i^{\rm EB}=
\sqrt{\frac{2\widehat v_i\log(4n/\delta_{\rm MC})}{J}}
+\frac{14H\log(4n/\delta_{\rm MC})}{3(J-1)}.
\]
The second term explicitly depends on the full return range even when the
sample variance vanishes. Independent response blocks, not paired transition
replicates from one response, determine the sample size. This bound is an
analytic alternative to the range-based term; the experimental records contain no evaluation on
the original-controller BabyAI responses, and no numerical
coverage or nonvacuity result is claimed for it. Testing it requires the
stored per-response block means and evaluation denominators.


\section{Implementation Details and Reproducibility Record}
\label{app:protocol}

The configuration below specifies the fallback, advisor interface, calibration splits, and budget accounting, followed by the compute environment, seed control, and logged quantities. These settings describe the original-controller experiments; the headroom-aware extension is reported separately in Appendix~\ref{app:headroom-ext}.

\subsection{Overview}
\label{app:impl-summary}

The deep fallback uses a Bayesian last layer (BLL) to approximate action-value uncertainty; ensemble value heads provide the $K$ continuation samples. BLL and ensemble estimates remain approximations, without established uniform action-value coverage. Calibration uses a fitting prefix, a disjoint calibration prefix, and an evaluation continuation. The original empirical radius fits residuals against paired-continuation proxies; the formal radius also includes the MC correction. Deterministic transitions do not make these proxies exact because advisor responses and continuation policies remain sampled. Appendix~\ref{app:protocol} lists the full configuration.

Empirical calibration uses 48 opportunities per seed in the earlier audit and 200 in the later groupwise audit; target miscoverage is $0.10$. We report 95\% confidence intervals (CIs) where the available seed records support them. Primary paired contrasts use bootstrap CIs and Holm correction within each declared family; independent replications retain separate families. Marginal CIs in return plots are distinct from paired-effect CIs. The extension's reported half-widths lack per-seed records and adjusted tests, so they do not establish familywise significance.


The primary configuration uses nominal price $\kappa=0.005$ for cost-adjusted return, $U=R-\kappa N_{\rm calls}/N_{\rm ep}$, and acquisition threshold $c_t=0.05$. Under $c_t=\lambda\kappa_t$, this operating point corresponds to $\lambda=10$. The price sweep uses $c_t=\kappa$, corresponding to $\lambda=1$. Matching nominal prices does not imply matching acquisition rules. Primary and price-sweep results retain their separate protocols.

\subsection{Fallback Learner and Posterior Approximation}

The tabular chain reports the number of states, horizon, transition kernel, reward location, prior, exact
advisor response distribution, cache key, budget, and seed list. Exact action gaps and response-contingent
values are computed by dynamic programming.

For each deep environment, the experimental configuration specifies observation preprocessing, action set, reward
function, episode horizon, training episodes, replay size, optimizer, learning rate, batch size, discount,
target-network update frequency, exploration schedule, feature dimension, posterior prior precision,
posterior update, and sample count \(K\). The per-environment configuration is:
\begin{center}
\footnotesize
\setlength{\tabcolsep}{3pt}
\renewcommand{\arraystretch}{0.96}
\textbf{Learner and architecture configuration}\par\smallskip
\begin{tabular}{@{}p{0.49\columnwidth}p{0.43\columnwidth}@{}}
\toprule
Parameter & Value\\
\midrule
Grid environments & GoToObj, GoToLocal\\
Observation & $7 \times 7 \times 3$ image + direction one-hot\\
Action space & 7 (left, right, forward, pickup, drop, toggle, done)\\
Episode horizon $H$ & 64\\
Training episodes & 60\\
Replay buffer size & 20{,}000\\
Optimizer & Adam, $\text{lr}=10^{-3}$\\
Batch size & 64\\
Discount $\gamma$ & 0.97\\

Target-network Polyak $\tau$ & 0.01\\
Exploration $\varepsilon$ & $1.0\to0.05$ over 70\% of episodes\\
Feature dimension & 64\\
Attention heads & 8\\
Hidden dimension & 128\\
\bottomrule
\end{tabular}

\medskip
\textbf{Query, calibration, and safety configuration}\par\smallskip
\begin{tabular}{@{}p{0.49\columnwidth}p{0.43\columnwidth}@{}}
\toprule
Parameter & Value\\
\midrule
Warm-up episodes & 20\\
Continuation lookahead $\ell$ & 6\\
Query budget $B$ (per episode) & 60\\
Primary query price $\kappa$ & $0.005$\\
Reported primary acquisition threshold $c_t$ & $0.05$\\
Query-price sweep & $\substack{\{0,0.002,0.005,\\ 0.01,0.02,0.05\}}$\\
Primary certificate tolerance $\epsilon_{\rm cert}$ & 0.25\\
Certificate-tolerance sweep & $\{0.05,0.1,0.25,0.5,1.0\}$\\
Certificate scale $C_{\rm cert}$ & 1.0\\
Calibration opportunities $n_{\rm cal}$ & 48; 200 for Mondrian calibration\\
Paired continuation rollouts $r_{\rm pair}$ & 3\\
Target miscoverage $\alpha$ & 0.10\\
MC failure level $\delta_{\rm MC}$ & 0.10\\
Corruption sweep & $\{0,0.25,0.5,0.75,1.0\}$\\
\bottomrule
\end{tabular}
\end{center}


\subsection{Advisor and Parser}

Each prompt contains the task instruction, current admissible observation summary, admissible actions, and
the minimal recent context required by the environment. The advisor is instructed to return one admissible
action without explanation for the action-only protocol. Structured-advice variants use a separately
specified schema. Parser failures produce \(\bot\) and are logged.

For repeat-query experiments, the protocol distinguishes an identical prompt, a randomized surface
paraphrase, and a context-augmented repair prompt. The query predictor observes which prompt family supplies
the response. This makes correlated failure measurable instead of assuming independent reliability.

\subsection{Calibration Splits}

Each training run is divided chronologically into a fitting prefix, a disjoint calibration prefix, and the
evaluation continuation. Response probabilities and confidence radii are calibrated only on the middle
prefix. A separate task-level split checks whether calibration transfers to unseen layouts and instructions.
No final evaluation return is used to choose a calibration method or threshold.

\subsection{Hard Budgets and Prices}

In the BabyAI experiments the hard budget is enforced \emph{per episode} (the guard $B_t>0$ of Algorithm~\ref{alg:voi} uses a
counter that resets at the start of each episode, $B{=}60$). The abstract budget-exhaustion analysis (Theorem~\ref{thm:allocation},
Proposition~\ref{prop:queries}) is stated for a single accounting horizon and applies to each episode's
budget. A shift experiment may instead use a prespecified token-bucket replenishment rule, stated
explicitly where used. The price \(\kappa_t\) is based on recorded tokens and serving cost for real
models and a normalized constant for simulated advisors. The Lagrange multiplier is either fixed from validation or
updated online by
\begin{equation}
 \lambda_{t+1}
 =\left[\lambda_t+\rho_t\left(q_t-\frac{B}{T}\right)\right]_+,
\end{equation}
with schedule \(\rho_t\) declared before evaluation.

Tables report calls summed over $60$ episodes and averaged over seeds. Thus $686$ calls correspond to $\approx11.4$ calls per episode, within the allowance of $B=60$.

Cost-adjusted return uses the same unit: mean per-episode return minus $\kappa$ times mean per-episode calls. This keeps reward and cost on a common per-episode scale.

\subsection{Reproducibility Record}
\label{sec:reproducibility}

This subsection records the compute environment, seed control, and logged quantities behind each reported result.

\paragraph{Compute environment.}
The original-controller real-language-model advisor rollouts use AWS SageMaker
\texttt{ml.g5.2xlarge} instances equipped with one NVIDIA A10G graphics processing unit (GPU) with 24\,GB of
GPU memory. In the original-controller study, GPU computation supports only the real Qwen2.5-1.5B and
Qwen2.5-7B advisor rollouts on BabyAI. The tabular experiments,
informative-advisor verification, predictor ablations, and held-out re-analyses
use only central processing units (CPUs).

\paragraph{Seed control.}
The original-controller real-advisor matrix uses seeds $0,\ldots,19$. The
certificate-tolerance frontier uses seeds $0,\ldots,9$, and the original-controller
query-price sweep uses seeds $0,\ldots,19$. For each run, the same integer seed
initializes Python's \texttt{random} module, NumPy, and PyTorch. Paired
comparisons reuse the same seed and warm-start state across the compared
methods.

The preceding subsections record the fallback architecture, optimizer, learning rate, replay settings, target-network schedule, posterior sample count, confidence level, query thresholds, dual update, cache key, and invalidation rule. Figures use logged experimental data.

Results distinguish real language models, simulated advisors, and replayed responses. The real-advisor logs
record request and response token counts, latency, parser outcome, cache status, query-value estimate,
calibration radius, certificate decision, executed action, and realized reward. Aggregate tables report CIs and multiple-comparison corrections where the
corresponding per-seed records are available.

\section{Real-Advisor Experimental Evaluation}
\label{sec:experiments}

Controlled tabular and pixel tasks expose known reliability and action-gap structure; the primary benchmark uses real language-model responses and measured call costs. Each study identifies its advisor and controller variant. Hyperparameter selection excludes evaluation returns.

%

\subsection{Environments and Advisor Protocols}

The controlled suite comprises a repeatable-context chain with exact action gaps and response reliabilities, and two partially observed pixel GridWorlds testing approximate uncertainty and advisor corruption. These provide diagnostic evidence.

The original controller uses real Qwen2.5-1.5B-Instruct and Qwen2.5-7B-Instruct advice on MiniGrid/BabyAI GoToObj and GoToLocal. Call reductions are observed at both advisor scales (Tables~\ref{tab:alg1-real},~\ref{tab:alg1-7b}). Its TextWorld/ALFWorld and Crafter extensions instead use simulated advisors with fixed reliability. Real-advisor ALFWorld results belong to the separately evaluated headroom-aware extension (Appendix~\ref{app:headroom-ext}).

The studies use different controller variants and sampling protocols. The original-controller matrix (Section~\ref{sec:babyai-alg1}; Tables~\ref{tab:alg1-real}, \ref{tab:matched-real},~\ref{tab:alg1-7b}) evaluates both BabyAI tasks with $1.5$B and $7$B advisors over $20$ seeds, one primary price, and the comparison set in Section~\ref{sec:baselines-note}.

The earlier surrogate-gate matrix (Section~\ref{sec:surrogate}, Table~\ref{tab:main-real}) uses one open-weight advisor and ten seeds. It evaluates the surrogate gate with and without the certificate, with threshold sweeps limited to competitive gates. Simulated-advisor studies add Crafter, ALFWorld, predictor ablations, and controlled corruption. Shift, repeated-error prompting, and broad calibration comparisons are targeted diagnostics on representative task--advisor pairs, rather than all combinations of methods.

\begin{table*}[t]
\centering
\small
\caption{Task families, evaluation roles, train/evaluation splits, episode horizons, advisors, parsing
rules, and independent seed counts for the original-controller evaluation.}
\resizebox{\textwidth}{!}{%
\begin{tabular}{lllllll}
\toprule
Task family & Role & Train/eval split & Horizon & Advisor(s) & Parser & Seeds\\
\midrule
MiniGrid/BabyAI & Core & GoToObj / GoToLocal (random layout) & $H{=}64$ & Qwen2.5-1.5B \& 7B & constrained action (7 admissible) & 20\\
TextWorld/ALFWorld & Extension & pick\_and\_place / look\_at\_obj (task split) & $H{=}50$ & simulated ($\eta{=}0.85$) & admissible command & 5\\
Crafter & Extension & procedural world seeds (train/eval split) & $H{=}200$ & simulated ($\eta{=}0.8$) & skill/action (17 admissible) & 5\\
\bottomrule
\end{tabular}}
\label{tab:tasks}
\end{table*}

The original-controller matrix uses Qwen2.5-1.5B-Instruct and Qwen2.5-7B-Instruct. Advisor records include model identifier, query date, decoding, prompt, tokens, latency, and monetary or normalized compute cost. The prespecified parser maps invalid, ambiguous, or out-of-vocabulary responses to \(\bot\), without ground-truth repair.

Reliability is evaluated in three forms. Natural reliability is whatever the fixed model and prompt produce.
Controlled corruption replaces a specified fraction of valid responses with plausible wrong actions.
Correlated-error trials reuse identical or paraphrased prompts after an initial wrong answer, measuring
whether repeated calls reproduce the same failure rather than behaving as independent Bernoulli draws.

\subsection{Response Prediction Without Leakage}

The learned predictor uses only information available before the current query: task instruction, current
observation embedding, cache metadata, previous responses, parser outcomes, and past realized rewards.
It predicts a distribution over parsed actions plus \(\bot\). The predictor is updated online after each paid
response. No future response, optimal action, test trajectory, or counterfactual return is used as a label.

We compare four predictors in the complete control loop: oracle response probabilities where the simulator
provides them, a feature-generalizing online predictor, a global-frequency predictor, and a uniform
predictor. For real advisors, the oracle row is replaced by a replay oracle computed only for retrospective
diagnosis and is never used to choose an online query.

\subsection{Baselines}
\label{sec:baselines-note}

Methods share fallback architecture, replay data, environment steps, and advisor. BabyAI comparisons include never-query, always-query, ASK-style uncertainty, random and early-schedule matched-call gates, and six controller variants: $\beta{=}0$, $\beta_{\rm conf}$, $\beta_{\rm form}$, certificate-off, global-frequency, and uniform prediction (Tables~\ref{tab:alg1-real}, \ref{tab:matched-real}). Matched-call gates test placement. Controlled suites additionally use fixed-period, novelty, ensemble, MC-dropout, VOE/Q-style, and RCMP-style gates; retrospective oracle placement is diagnostic. Torrey--Taylor importance advising, A7 reuse, and SayCan-style always-ranking are related designs, not evaluated baselines.

Our controller is compared at matched realized calls against the random and scheduled placement gates,
and the surrogate cost frontier is swept over query prices. Hyperparameters are
tuned on validation tasks or a held-out training prefix with the same search budget within each comparison.

\subsection{Metrics and Statistical Protocol}

We report undiscounted task return, success rate, learning-curve area, environment steps to a prespecified
target, advisor calls, tokens, latency, parser-failure rate, and cost-adjusted return
\(R-\sum_t\kappa_tq_t\). Tokens and latency are recorded per call and reported in the surrogate
real-advisor table (Table~\ref{tab:main-real}) and the experimental records; because they are a fixed
per-call function of the prompt and are proportional to the call counts already shown, the original-controller
tables (Tables~\ref{tab:alg1-real}--\ref{tab:matched-real}) report calls and cost-adjusted
return and do not duplicate token/latency columns. Allocation quality is measured by the realized improvement attributable to each
query in simulators where counterfactual rollouts are available, precision among the top \(B\) query
opportunities, and regret relative to oracle myopic placement.

Safety metrics include total advice acceptance, correct-advice rejection, wrong-advice acceptance, the
return loss caused by accepted wrong advice, and the loss prevented by the certificate. Calibration metrics
include Brier score and expected calibration error (ECE) for the response predictor and query-value event, marginal per-action coverage,
simultaneous state-action coverage, interval width, and coverage conditional on shift.

Controlled tasks use between 20 and 60 independent seeds, as stated for each study (24 for the clonable chain and 60 for the predictor chain); the surrogate-gate core and extension real-advisor tasks use
10 seeds per task and model, and the original controller
(Section~\ref{sec:babyai-alg1}) uses 20. Return tables report mean with a 95\%
CI wherever the record includes per-seed spread; the stated exceptions are the BabyAI $\beta_t$
calibration audit (Table~\ref{tab:beta-babyai}), which reports medians, interquartile ranges (IQRs), and mean/worst-seed
coverage of the per-seed radii rather than a return CI, and the five-seed controlled deep-RL pilots in
Appendix~\ref{app:controlled}, whose means are explicitly not used for significance claims.
Primary comparisons use paired seed/task bootstrap intervals with Holm correction
within each declared family. The target return, statistical families, and evaluation checkpoints are fixed
before inspecting final runs.

\paragraph{Multiplicity for the six extension contrasts.}
If all six contrasts in Table~\ref{tab:headroom-placement} form one family at level $\alpha=0.05$, Holm's step-down procedure compares ordered raw $p$-values with $0.05/(6-k+1)$ at rank $k=1,\ldots,6$ and stops at the first failed comparison. These thresholds specify a procedure, not observed $p$-values. The retained aggregate half-widths do not support this analysis without paired records and the corresponding inferential protocol.

\subsection{Prespecified Evaluation Protocols}
\label{app:prespecified-protocols}

This subsection specifies protocols for prospective evaluation.

\paragraph{Prespecified action-relevant calibration protocol.}
A proposed stronger test first freezes the query predictor, fallback, cache
policy, and decision-bound estimator on a training split. The randomized
calibration collector must then obtain labels on opportunities sampled
without the prospective query gate and retain sufficient probability mass
in three disjoint, pre-query strata: (a) predicted branch changes near the
purchase threshold, (b) predicted branch changes far from threshold, and
(c) no predicted branch change. Subdividing again by task and training
phase tests whether one model quietly borrows easy, near-zero residuals
from another task. Report $n_{\rm cal}$ and $n_{\rm eval}$ separately in
every stratum, action-change frequency on both sides of the split,
conditional coverage and its binomial interval, simultaneous trajectory
coverage, the empirical and formal radii, and abstention rates. An
undercovered stratum is a negative result rather than evidence that the
calibration theorem holds ``approximately.''

\paragraph{Same-run matched acquisition protocol.}
For each seed and advisor, save an identical initial learner checkpoint.
The controller and an equal-budget early schedule should be evaluated
within the same experiment, with the same admissible advisor interface,
learning steps, decoder, response budget, and validation information.
The early schedule must not inherit a hindsight budget from the test
controller unless both are also compared with a budget selected on held-out
training tasks. Record the realized calls per seed rather than matching
only averages. The unit of randomization for return inference is an
independent environment/learner seed, not a repeated episode under the same
trained learner. Predeclare the primary family of paired contrasts, compute
raw per-seed differences and Holm-adjusted tests on those differences, and
report CIs that follow the selected estimator. The
original-controller matched run does not satisfy the proposed positive-placement
criterion, and no such claim is made.

\paragraph{End-to-end computation and task utility.}
Time the entire acquisition stage (feature extraction, response
prediction, posterior draws, calibration lookup, and checking the cache)
per environment decision, including decisions that do not issue a call.
Log call latency, inference energy, parsing cost, calibration cost, and the
number of future episodes over which calibration is amortized. The
comparison with a no-query fallback should report net task utility and
net latency, not only the number of advisor requests. The headroom-aware extension reports task success rates and real-advisor ALFWorld results
(Appendix~\ref{app:headroom-ext}), while the original controller still lacks
these measurements. A near-oracle reference return, a longer training horizon
with a shared exploration schedule, and complete per-seed paired records remain
important for establishing wider deployment value.

\subsection{Original Controller with Real Qwen Advisors}
\label{sec:babyai-alg1}

We evaluate the original controller (Algorithm~\ref{alg:voi}) with Qwen2.5-1.5B on both BabyAI tasks (20 seeds, 60 episodes, shared warm start). It uses online response predictor $\hat p_t(y\mid z)$, $K$ ensemble continuation samples, prospective decisions $d_t(y)$, and estimator $\hat\Delta_t$ (Eq.~\eqref{eq:query-estimator}). Unlike the tabular instantiation, state-dependent ensemble spread $c\cdot\mathrm{sd}$ gives discriminating bounds $L_t,U_t$. The query estimator remains an ensemble approximation.

Calibration uses a low-variance paired-continuation proxy on cloned environments. We compare $\beta{=}0$, empirical split-conformal $\beta_{\rm conf}$ (mean $\approx2\times10^{-3}$), and formal $\beta_{\rm form}$, which adds Eq.~\eqref{eq:mc-radius}'s MC correction. Deterministic transitions do not make the proxy exact: advisor decoding and continuation policies remain sampled. Only $\beta_{\rm form}$ supports the nonnegative-query-value guarantee under Assumption~\ref{ass:value-error} and Theorem~\ref{thm:allocation}; $\beta_{\rm conf}$ remains an empirical predictive radius.

\begin{table*}[t]
\centering\small
\setlength{\tabcolsep}{4pt}
\caption{Additional original-controller metrics for Qwen2.5-1.5B (20 seeds, 60 episodes); return and call volume appear in Figure~\ref{fig:return-calls}. Paired differences are against never-query, with unadjusted 95\% half-widths where retained; bold marks intervals excluding zero. Cost-adjusted return uses price 0.005 per call per episode, and certificate acceptance is the recorded fraction.}
\label{tab:alg1-real}
\begin{tabular}{llccc}
\toprule
Task & Method & Paired vs never & Cost-adj. $\uparrow$ & Cert. accept\\
\midrule
GoToObj & Never query & \textendash{} & $0.093$ & \textendash{}\\
 & Always query & $-0.037$ (sig) & $-0.228$ & \textendash{}\\
 & ASK / uncertainty & $-0.007 \pm 0.008$ & $0.040$ & $0.08$\\
 & Ours, $\beta{=}0$ & $\mathbf{+0.024 \pm 0.018}$ & $0.105$ & $0.79$\\
 & \textbf{Ours, $\beta_{\rm conf}$} & $\mathbf{+0.029 \pm 0.016}$ & $\mathbf{0.116}$ & $0.78$\\
 & Ours, $\beta_{\rm form}$ & $+0.000 \pm 0.000$ & $0.093$ & \textendash{}\\
GoToLocal & Never query & \textendash{} & $0.124$ & \textendash{}\\
 & Always query & $-0.069$ (sig) & $-0.230$ & \textendash{}\\
 & ASK / uncertainty & $-0.013 \pm 0.015$ & $0.062$ & $0.13$\\
 & Ours, $\beta{=}0$ & $-0.017 \pm 0.023$ & $0.094$ & $0.76$\\
 & Ours, $\beta_{\rm conf}$ & $-0.007 \pm 0.023$ & $0.111$ & $0.69$\\
 & Ours, $\beta_{\rm form}$ & $+0.000 \pm 0.000$ & $0.124$ & \textendash{}\\
\bottomrule
\end{tabular}
\end{table*}

\paragraph{Placement versus volume.}
A return improvement over never-query does not isolate query placement. An early or random schedule might achieve the same return with the same number of calls. We therefore match realized call budgets. The early-schedule gate spends $\beta_{\rm conf}$'s per-seed count at the earliest steps. The random gate targets the same per-episode budget using an online select-$k$-of-remaining rule. It reaches the target in episodes that run to the horizon and falls slightly short when episodes end early.

Component ablations retain query placement while disabling the execution certificate (\texttt{cert-off}) or replacing the online predictor with global-frequency (\texttt{global-pred}) or uniform (\texttt{unif-pred}) predictions. Both $1.5$B and $7$B Qwen advisors use $20$ seeds and $60$ episodes. Table~\ref{tab:matched-real} reports the complete comparison.

\begin{table*}[t]
\centering
\small
\caption{Placement and component isolation for the original controller ($\beta_{\rm conf}$) with $1.5$B and $7$B Qwen advisors ($20$ seeds, $60$ episodes). Matched gates run at $\beta_{\rm conf}$'s realized call budget; component ablations keep $\beta_{\rm conf}$'s placement. ``$\Delta$ vs ours'' is the \emph{unadjusted} paired per-seed return of $\beta_{\rm conf}$ minus that baseline; bold marks an unadjusted CI excluding zero. Holm correction covers the $20$ comparisons ($5$ baselines $\times\,2$ tasks $\times\,2$ advisor sizes).}
\resizebox{\textwidth}{!}{%
\begin{tabular}{lll cc cc}
\toprule
& & & \multicolumn{2}{c}{Qwen-1.5B} & \multicolumn{2}{c}{Qwen-7B}\\
\cmidrule(lr){4-5}\cmidrule(lr){6-7}
Task & Baseline & Isolates & Return $\uparrow$ & $\Delta$ vs ours & Return $\uparrow$ & $\Delta$ vs ours\\
\midrule
GoToObj & Ours, $\beta_{\rm conf}$ & \textendash{} & $0.122{\pm}0.019$ & \textendash{} & $0.123{\pm}0.021$ & \textendash{}\\
 & Random matched & placement & $0.098{\pm}0.018$ & $\mathbf{+0.024{\pm}0.014}$ & $0.101{\pm}0.018$ & $\mathbf{+0.022{\pm}0.022}$\\
 & Schedule matched & placement & $0.117{\pm}0.017$ & $+0.005{\pm}0.016$ & $0.120{\pm}0.020$ & $+0.003{\pm}0.023$\\
 & Ours, cert-off & certificate & $0.110{\pm}0.019$ & $+0.012{\pm}0.015$ & $0.108{\pm}0.022$ & $+0.016{\pm}0.020$\\
 & Ours, global-pred & predictor & $0.110{\pm}0.022$ & $+0.013{\pm}0.017$ & $0.112{\pm}0.022$ & $+0.011{\pm}0.018$\\
 & Ours, unif-pred & predictor & $0.108{\pm}0.027$ & $+0.014{\pm}0.019$ & $0.109{\pm}0.029$ & $+0.015{\pm}0.024$\\
\midrule
GoToLocal & Ours, $\beta_{\rm conf}$ & \textendash{} & $0.116{\pm}0.019$ & \textendash{} & $0.117{\pm}0.018$ & \textendash{}\\
 & Random matched & placement & $0.119{\pm}0.026$ & $-0.003{\pm}0.023$ & $0.115{\pm}0.027$ & $+0.002{\pm}0.025$\\
 & Schedule matched & placement & $0.099{\pm}0.021$ & $+0.017{\pm}0.023$ & $0.102{\pm}0.020$ & $+0.015{\pm}0.022$\\
 & Ours, cert-off & certificate & $0.110{\pm}0.019$ & $+0.006{\pm}0.018$ & $0.113{\pm}0.018$ & $+0.005{\pm}0.016$\\
 & Ours, global-pred & predictor & $0.102{\pm}0.016$ & $+0.015{\pm}0.018$ & $0.106{\pm}0.014$ & $+0.011{\pm}0.018$\\
 & Ours, unif-pred & predictor & $0.109{\pm}0.015$ & $+0.008{\pm}0.019$ & $0.107{\pm}0.012$ & $+0.010{\pm}0.017$\\
\bottomrule
\end{tabular}%
}
\label{tab:matched-real}
\end{table*}

\begin{table*}[t]
\centering\small
\setlength{\tabcolsep}{4pt}
\caption{Absolute returns and realized calls in the fresh matched-budget run (20 seeds); paired contrasts appear in Figure~\ref{fig:placement-calibration}(a). These means belong to this replication and are not combined with the primary matrix. Random placement matches the gate's cumulative calls exactly; early-schedule counts differ slightly.}
\label{tab:exact-random}
\begin{tabular}{llc}
\toprule
Task (advisor) & Policy & Return $\uparrow$ (calls)\\
\midrule
GoToObj (1.5B) & Ours, $\beta_{\rm conf}$ & $\mathbf{0.116}$ ($87.2$)\\
 & random\_exact & $0.097$ ($87.2$)\\
 & schedule matched & $0.106$ ($86.8$)\\
GoToObj (7B) & Ours, $\beta_{\rm conf}$ & $\mathbf{0.108}$ ($100.5$)\\
 & random\_exact & $0.106$ ($100.5$)\\
 & schedule matched & $0.100$ ($99.5$)\\
GoToLocal (1.5B) & Ours, $\beta_{\rm conf}$ & $0.103$ ($59.9$)\\
 & random\_exact & $\mathbf{0.118}$ ($59.9$)\\
 & schedule matched & $0.102$ ($59.5$)\\
GoToLocal (7B) & Ours, $\beta_{\rm conf}$ & $0.102$ ($55.4$)\\
 & random\_exact & $\mathbf{0.114}$ ($55.4$)\\
 & schedule matched & $0.105$ ($54.5$)\\
\bottomrule
\end{tabular}
\end{table*}

At the empirical-radius gate's realized
budget, the early schedule shows no
resolved return difference on either task.
Certificate removal and the global-frequency
and uniform predictor ablations also
produce no resolved improvements.
Against random placement, the GoToObj
gate-minus-random differences are
$+0.024\pm0.014$ at 1.5B and
$+0.022\pm0.022$ at 7B.
Only the 1.5B contrast survives Holm
correction over the 20-comparison family
(adjusted $p=0.010$). However, the random
baseline spends fewer calls. At 1.5B,
the predictive gate makes 74.8 and 59.1
calls on GoToObj and GoToLocal, versus
68.5 and 52.5 for random placement.
These differences correspond to random
call deficits of 8.4\% and 11.2\%.
Consequently, the primary comparison
does not isolate a state-selection
advantage from query-volume differences.

A fresh $20$-seed replication carries call deficits forward with \texttt{random\_exact}, matching $\beta_{\rm conf}$ cumulatively to within one call (Table~\ref{tab:exact-random}). Its absolute returns remain separate from Table~\ref{tab:matched-real}. GoToObj gate-minus-random differences are $+0.019\pm0.018$ at $1.5$B ($87.2$ versus $87.2$ calls) and $+0.003\pm0.017$ at $7$B ($100.5$ versus $100.5$). Neither establishes a Holm-significant placement effect in this family. Schedule contrasts are $+0.010\pm0.018$ at $1.5$B and $+0.009\pm0.017$ at $7$B; GoToLocal differences also remain within uncertainty. The early schedule itself improves GoToObj over never-query by paired $+0.024\pm0.017$ and $+0.027\pm0.021$ at the two scales, so the original controller does not uniquely capture the gain in Table~\ref{tab:alg1-real}.

Together, the comparisons support economical advice volume without establishing better state selection than an early schedule. GoToLocal still uses $59$--$65$ calls despite no resolved benefit, so task-level utility detection remains unestablished. Certificate and predictor ablations provide no resolved real-task return contribution; their supported effects concern corrupted-advice screening (Section~\ref{sec:babyai-cert}) and controlled prediction.

\paragraph{Evaluation with Qwen2.5-7B.}
With Qwen-$7$B and $20$ seeds, always-query ($\sim3{,}400$ calls) reduces return by $-0.036\pm0.017$, and GoToLocal shows no resolved benefit ($-0.006\pm0.022$). The GoToObj effect (Section~\ref{sec:return}) spans both advisor sizes but remains comparable to early scheduling, without isolating the proposed placement mechanism (Table~\ref{tab:alg1-7b}).

\begin{table*}[t]
\centering\small
\setlength{\tabcolsep}{4pt}
\caption{Additional original-controller metrics for Qwen2.5-7B (20 seeds, 60 episodes); return and call volume appear in Figure~\ref{fig:return-calls}. Paired differences are against never-query, with unadjusted 95\% half-widths where retained; bold marks intervals excluding zero. Cost-adjusted return uses price 0.005 per call per episode, and certificate acceptance is the recorded fraction.}
\label{tab:alg1-7b}
\begin{tabular}{llccc}
\toprule
Task & Method & Paired vs never & Cost-adj. $\uparrow$ & Cert. accept\\
\midrule
GoToObj & Never query & \textendash{} & $0.093$ & \textendash{}\\
 & Always query & $-0.036$ (sig) & $-0.226$ & $1.00$\\
 & ASK / uncertainty & $-0.007 \pm 0.008$ & $0.040$ & $0.08$\\
 & Ours, $\beta{=}0$ & $\mathbf{+0.026 \pm 0.016}$ & $0.107$ & $0.61$\\
 & \textbf{Ours, $\beta_{\rm conf}$} & $\mathbf{+0.030 \pm 0.015}$ & $\mathbf{0.116}$ & $0.71$\\
 & Ours, $\beta_{\rm form}$ & $+0.000 \pm 0.000$ & $0.093$ & \textendash{}\\
GoToLocal & Never query & \textendash{} & $0.124$ & \textendash{}\\
 & Always query & $-0.062$ (sig) & $-0.221$ & $1.00$\\
 & ASK / uncertainty & $-0.013 \pm 0.015$ & $0.062$ & $0.12$\\
 & Ours, $\beta{=}0$ & $-0.013 \pm 0.022$ & $0.097$ & $0.71$\\
 & Ours, $\beta_{\rm conf}$ & $-0.006 \pm 0.022$ & $0.112$ & $0.66$\\
 & Ours, $\beta_{\rm form}$ & $+0.000 \pm 0.000$ & $0.124$ & \textendash{}\\
\bottomrule
\end{tabular}
\end{table*}

\paragraph{Untuned query volume.}
We compare the default controller against five fixed, front-loaded budgets of $25$, $50$, $75$, $100$, and $150$ calls on five BabyAI tasks with both Qwen advisors (Table~\ref{tab:volume-real}). All methods share $20$ seeds and $60$ episodes. The best fixed budget is selected separately for each task in hindsight and is therefore an oracle comparator, not an executable policy without tuning data. At $1.5$B, the best fixed volume is $50$ on GoToObj and Open, $75$ on GoToObjDoor and Pickup, and $150$ on GoToLocal. The preferred volume changes at $7$B, including $100$ on GoToLocal and $150$ on Pickup.

Query counts span $59$--$224$ at $1.5$B and $65$--$218$ at $7$B, diverging from hindsight-optimal budgets: $1.5$B uses $157$ calls on Open versus the best fixed $50$, and $224$ on Pickup versus $75$. Returns lie within the best schedule's marginal $95\%$ uncertainty across both scales, without establishing equivalence. The $7$B Open and Pickup means trail both best-fixed and never-query. GoToObj gains over never-query span both scales ($0.123$ versus $0.093$ at $7$B), but optimal-volume identification and uniform cost-aware improvement remain unestablished. The fixed-volume return surface is flat relative to the per-seed CI.

\paragraph{Held-out selection of a global call budget.}
For a deployable alternative, we select one fixed early-advice budget on four tasks and evaluate it on the fifth, repeating for every held-out task. Mean return regret relative to the per-task hindsight best fixed schedule is $-0.0006$ (1.5B) and $+0.0017$ (7B) for the controller, versus $+0.0078$ for the validation-selected global fixed schedule. Negative regret is possible because the controller is not a member of the fixed-budget comparison class. With cost-adjusted utility $U=R-0.005N_{\rm calls}/N_{\rm ep}$, the controller's mean regrets are $0.0042$ and $0.0067$, compared with $0.0063$ and $0.0098$ for the validated global schedule. The controller overspends on some low-headroom tasks, however, and the five held-out tasks do not support a population uncertainty interval. These comparisons establish average descriptive competitiveness, not a uniform return or utility advantage.

\begin{table*}[t]
\centering
\small
\caption{Robustness to untuned query volume on five BabyAI tasks at both advisor sizes ($20$ seeds, $60$ episodes, clean advisor). Columns \texttt{v25}--\texttt{v150} give mean return for fixed front-loaded call schedules of that size; ``best fixed'' is the top schedule per task, \emph{selected in hindsight}. $\beta_{\rm conf}$ reports the controller's return and its realized, untuned call count. All return entries are higher-is-better.}
\begin{tabular}{lccccc cc c}
\toprule
Task & v25 & v50 & v75 & v100 & v150 & best fixed & $\beta_{\rm conf}$ (calls) & never\\
\midrule
\multicolumn{9}{l}{\emph{Qwen2.5-1.5B advisor}}\\
GoToObj     & $0.111$ & $0.114$ & $0.105$ & $0.099$ & $0.101$ & v50 $0.114$   & $\mathbf{0.122}$ ($75$)  & $0.093$\\
GoToLocal   & $0.105$ & $0.112$ & $0.110$ & $0.118$ & $0.121$ & v150 $0.121$  & $0.116$ ($59$)  & $\mathbf{0.123}$\\
GoToObjDoor & $0.097$ & $0.093$ & $\mathbf{0.109}$ & $0.100$ & $0.101$ & v75 $\mathbf{0.109}$   & $0.104$ ($122$) & $0.104$\\
Open        & $0.015$ & $\mathbf{0.016}$ & $0.013$ & $0.014$ & $0.011$ & v50 $\mathbf{0.016}$   & $\mathbf{0.016}$ ($157$) & $0.012$\\
Pickup      & $0.006$ & $0.006$ & $0.009$ & $0.007$ & $0.008$ & v75 $0.009$   & $\mathbf{0.013}$ ($224$) & $0.011$\\
\midrule
\multicolumn{9}{l}{\emph{Qwen2.5-7B advisor}}\\
GoToObj     & $0.101$ & $0.111$ & $0.109$ & $0.100$ & $0.098$ & v50 $0.111$   & $\mathbf{0.123}$ ($82$)  & $0.093$\\
GoToLocal   & $0.103$ & $0.104$ & $0.112$ & $\mathbf{0.123}$ & $0.119$ & v100 $\mathbf{0.123}$  & $0.117$ ($65$)  & $\mathbf{0.123}$\\
GoToObjDoor & $0.097$ & $0.093$ & $\mathbf{0.110}$ & $0.097$ & $0.095$ & v75 $\mathbf{0.110}$   & $0.105$ ($138$) & $0.104$\\
Open        & $0.012$ & $\mathbf{0.018}$ & $0.015$ & $0.008$ & $0.011$ & v50 $\mathbf{0.018}$   & $0.011$ ($218$) & $0.012$\\
Pickup      & $\mathbf{0.011}$ & $0.008$ & $0.005$ & $0.006$ & $\mathbf{0.011}$ & v150 $\mathbf{0.011}$  & $0.008$ ($102$) & $\mathbf{0.011}$\\
\bottomrule
\end{tabular}
\label{tab:volume-real}
\par\smallskip
\begin{minipage}{\linewidth}
\footnotesize
\emph{Note.} The GoToLocal never-query mean is $0.123$
in this five-task volume panel and $0.124$
in the primary two-task evaluation.
We retain these panel-specific reported
means and do not construct paired
comparisons across the panels.
\end{minipage}
\end{table*}

\paragraph{Longer-horizon learning.}
A separate $500$-episode GoToObj diagnostic averages the last ten episodes at checkpoints $60$, $120$, $250$, and $500$. At $60$, returns are $0.159\pm0.044$ for the controller and $0.144\pm0.032$ for never-query. At $500$, the paired gain is $0.016\pm0.022$, and the early-schedule difference is $0.014\pm0.019$. The primary $60$-episode result instead averages cumulative return ($0.122$ versus $0.093$). Exploration decays over $70\%$ of $500$ rather than $70\%$ of $60$, so rates differ at episode $60$. The diagnostic suggests diminishing advice value but does not provide a controlled continuation.



\subsection{Execution Screening Under Corrupted Advice}
\label{sec:babyai-cert}

The gate controls acquisition; the certificate controls execution (Eq.~\eqref{eq:post-response-action}). We replace valid responses with plausible wrong actions at $\text{CORRUPT\_FRAC}=0.5$ over $8$ seeds on both tasks. The original controller with $\beta_{\rm conf}$ accepts correct advice at $0.65$ and corrupted advice at $0.48$ (paired gap $+0.171\pm0.152$, excluding zero at $95\%$). The uncalibrated $\beta{=}0$ controller shows no resolved discrimination ($0.50$ versus $0.46$, gap $+0.044\pm0.152$). This comparison concerns the $\beta_{\rm conf}$ variant, not the certificate form alone.

These rates confound screening with query-state selection: $\beta_{\rm conf}$ and $\beta{=}0$ select different states, and the $0.17$ discrimination gap has wide uncertainty. At fixed $\beta_{\rm conf}$ query opportunities, switching the certificate on reduces corrupted acceptance from $1.00$ to $0.47$--$0.48$ and correct acceptance to $0.64$--$0.67$. Switching it off raises return by paired $+0.022\pm0.022$ on GoToObj and $+0.026\pm0.015$ on GoToLocal, with both intervals excluding zero. At $\text{CORRUPT\_FRAC}=0.5$, false rejections outweigh the screening benefit. With clean advice (\texttt{corrupt}${=}0$), on-minus-off differences are $+0.006\pm0.027$ and $+0.003\pm0.016$. The certificate screens corrupted advice without establishing return benefit or formal safety; tighter bounds need to retain more useful suggestions.

\paragraph{Tolerance and corruption sweep.}
We sweep $\epsilon\in\{0.05,0.1,0.25,0.5,1.0\}$ against corruption $\{0,0.25,0.5,0.75,1.0\}$ on GoToObj with real $1.5$B advice ($10$ seeds), recalibrating per $\epsilon$ (Table~\ref{tab:eps-frontier}). At $\epsilon{=}0.05$, no action satisfies $L_t(a_c)\ge\max_a U_t(a)-\epsilon$. Prospective decisions fall back, $\widehat\Delta_t$ drops below $\kappa$, and no queries occur. Cert-on and cert-off coincide with never-query ($\Delta$ret $=0$); zero corrupted executions therefore reflect abstention, not screening. With queries, cert-off executes $100\%$ of corrupted advice. At corrupt${=}0.5$, cert-on acceptance is $0\%$, $\approx47\%$, and $\approx72\%$ at $\epsilon{=}0.1$, $0.25$, and $0.5$. Tightening generally lowers acceptance but is not strictly monotone: at corrupt${=}0.25$, acceptance is $0.50$ at $\epsilon{=}0.1$ and $0.44$ at $\epsilon{=}0.25$.
None of the 20 certificate-on-minus-off
return contrasts has a positive unadjusted
95\% interval excluding zero.
Four negative contrasts exclude zero before
multiplicity adjustment: $\epsilon=0.50$
at corruption $0.75$, and $\epsilon=1.00$
at corruption $0$, $0.50$, and $0.75$.
None survives Holm correction across the
20 nondegenerate comparisons.
The positive point estimates of $+0.024$
and $+0.031$ at corruption $0.25$ also
have intervals that include zero.
Thus, the tolerance sweep demonstrates
corrupted-advice screening, but not a
statistically established improvement
in task return.

\begin{table*}[!tp]
\centering
\footnotesize
\caption{Certificate tolerance and response corruption on GoToObj with real Qwen2.5-1.5B advice (10 seeds, recalibrated per tolerance; the certificate-off arm shares query opportunities). Wrong acceptance is conditional on a queried corrupted response; denominators and intervals are unavailable. Return differences are paired on-minus-off means with unadjusted 95\% half-widths; bold marks intervals that exclude zero before adjustment. ``n/l'' denotes an unlogged value.}
\begin{tabular}{lccccc}
\toprule
 & \multicolumn{5}{c}{corruption rate}\\
\cmidrule(lr){2-6}
$\epsilon$ & $0.0$ & $0.25$ & $0.5$ & $0.75$ & $1.0$\\
\midrule
\multicolumn{6}{l}{\emph{wrong-advice acceptance $\downarrow$ (cert on; cert-off${=}1.00$ where queries occur; over $10$ seeds)}}\\
$0.05$ & \textendash{} & $\mathrm{N/A}^\dagger$ & $\mathrm{N/A}^\dagger$ & $\mathrm{N/A}^\dagger$ & $\mathrm{N/A}^\dagger$\\
$0.10$ & \textendash{} & $0.50$ & $0.00$ & $0.17$ & $0.42$\\
$0.25$ & \textendash{} & $0.44$ & $0.47$ & $0.47$ & $0.44$\\
$0.50$ & \textendash{} & $0.66$ & $0.72$ & $0.69$ & $0.68$\\
$1.00$ & \textendash{} & n/l & n/l & n/l & n/l\\
\midrule
\multicolumn{6}{l}{\emph{$\Delta$ret = cert-on $-$ cert-off (paired $95\%$ CI half-width in parentheses)}}\\
$0.05$ & $0.000\,(.00)$ & $0.000\,(.00)$ & $0.000\,(.00)$ & $0.000\,(.00)$ & $0.000\,(.00)$\\
$0.10$ & $-.005\,(.007)$ & $+.000\,(.007)$ & $+.007\,(.010)$ & $+.007\,(.011)$ & $-.004\,(.009)$\\
$0.25$ & $+.008\,(.013)$ & $-.004\,(.015)$ & $-.011\,(.016)$ & $-.010\,(.015)$ & $+.016\,(.018)$\\
$0.50$ & $-.017\,(.019)$ & $+.024\,(.026)$ & $-.017\,(.021)$ & $\mathbf{-.037\,(.030)}$ & $+.002\,(.024)$\\
$1.00$ & $\mathbf{-.023\,(.021)}$ & $+.031\,(.033)$ & $\mathbf{-.023\,(.021)}$ & $\mathbf{-.032\,(.028)}$ & $-.005\,(.023)$\\
\bottomrule
\end{tabular}
\\[2pt]
{\footnotesize $^\dagger$ No queries occur
at $\epsilon=0.05$. Conditional
corrupted-advice acceptance is therefore
undefined; both policies coincide with
the never-query behavior.}
\label{tab:eps-frontier}
\end{table*}

\subsection{Response to Query Price}
\label{sec:price-app}

\paragraph{Original-controller price sweep.}
We sweep the query price $\kappa\in\{0,0.002,0.005,0.01,0.02,
0.05\}$ inside the original controller (Algorithm~\ref{alg:voi}) on three BabyAI tasks with the real $1.5$B advisor ($20$ seeds). Because $\kappa$ enters only the gate rule
$\widehat\Delta_t-\beta_t\ge\kappa$ and not calibration, we calibrate once per seed and rerun the gate at
each price.
Realized query volume decreases monotonically
as the price rises. On GoToObj, the
controller makes 4.16 calls per episode
at zero price, 0.56 at $\kappa=0.01$,
0.03 at $\kappa=0.02$, and none at
$\kappa=0.05$. GoToLocal and
GoToObjDoor exhibit the same decreasing
pattern. At $\kappa=0.05$, every task
reaches zero calls and matches its
never-query return. At $\kappa=0.02$,
GoToLocal retains 0.003 calls per
episode and therefore has not reached
exact abstention.

The sweep reports marginal return intervals
but lacks paired return and net-utility
contrasts against never-query.
Although GoToObj's reported cost-adjusted
utility is numerically above its no-advice
reference at several prices, these
aggregates do not establish statistical
superiority or equivalence.
The primary GoToObj comparison in
Section~\ref{sec:return} uses a separate
evaluation and is not combined with
the price sweep. These experiments
establish behavioral sensitivity to
query price, while leaving net
serving-time and utility improvements
unresolved.

\paragraph{Surrogate-gate frontier.}
The surrogate gate of Section~\ref{sec:surrogate} is swept separately (Table~\ref{tab:cost-real}, Figure~\ref{fig:frontier-real}). The query cost enters Eq.~\eqref{eq:query-rule}; consequently, each value of \(\kappa\) requires a new
controller run. Post hoc subtraction of \(\kappa N_{\rm calls}\) from one fixed trajectory is reported only
as repricing, not as behavioral cost sensitivity. For every method we sweep the gate threshold and plot all
non-dominated return--call pairs. The frontier contains enough thresholds to expose nonmonotonic training
effects rather than connecting two or three isolated operating points.

\begin{table*}[t]
\centering\small
\setlength{\tabcolsep}{4pt}
\caption{Net utility and recorded Pareto status of the surrogate gate on GoToObj; return and call volume appear in Figure~\ref{fig:frontier-real}. Each price is a separately trained run without paired utility intervals. These runs are separate from the original-controller sweep in Figure~\ref{fig:query-price-response}.}
\label{tab:cost-real}
\begin{tabular}{llccc}
\toprule
Task & Method & Price & Net utility $\uparrow$ & Pareto efficient?\\
\midrule
BabyAI-GoToObj & ASK / uncertainty & 0.01 & $-0.528$ & no\\
 & \textbf{Ours} & 0.01 & $-0.064$ & no\\
 & ASK / uncertainty & 0.02 & $-1.102$ & no\\
 & \textbf{Ours} & 0.02 & $-0.022$ & no\\
 & ASK / uncertainty & 0.05 & $-2.825$ & no\\
 & \textbf{Ours} & 0.05 & $\mathbf{0.204}$ & \textbf{yes}\\
 & ASK / uncertainty & 0.10 & $-5.696$ & no\\
 & \textbf{Ours} & 0.10 & $\mathbf{0.210}$ & \textbf{yes}\\
\bottomrule
\end{tabular}
\end{table*}

\begin{figure}[t]
\centering
\begin{tikzpicture}
\begin{axis}[
  width=0.92\columnwidth,
  height=5.0cm,
  title={(a) Full range},
  title style={font=\small},
  xlabel={Realized LLM calls},
  ylabel={Mean episodic return},
  xmin=-100, xmax=3800,
  ymin=0.0, ymax=0.30,
  grid=major,
  grid style={gray!30},
  legend pos=south east,
  legend cell align={left},
  legend style={font=\scriptsize},
  tick label style={font=\scriptsize},
  label style={font=\small},
]
\addplot[black, thick, dashed, mark=none] coordinates {
  (0.0, 0.1875)
  (1.2, 0.2117)
  (27.2, 0.2265)
};
\addlegendentry{Pareto frontier}
\addplot[
  only marks, mark=*, mark size=2.5pt,
  blue!80!black, thick,
  error bars/.cd, y dir=both, y explicit,
] coordinates {
  (1516.7, 0.1889) +- (0, 0.0290)
  (690.8, 0.2082) +- (0, 0.0178)
  (27.2, 0.2265) +- (0, 0.0486)
  (1.2, 0.2117) +- (0, 0.0317)
};
\addlegendentry{Ours (surrogate VOI gate)}
\addplot[
  only marks, mark=square*, mark size=2.2pt,
  red!80!black, thick,
  error bars/.cd, y dir=both, y explicit,
] coordinates {
  (3445.5, 0.0463) +- (0, 0.0291)
};
\addlegendentry{Uncertainty gate}
\addplot[
  only marks, mark=diamond*, mark size=3pt,
  green!50!black, thick,
  error bars/.cd, y dir=both, y explicit,
] coordinates {
  (0.0, 0.1875) +- (0, 0.0380)
};
\addlegendentry{Never-query}
\node[above right, font=\tiny, blue!70!black] at (axis cs:1516.7,0.1889) {$\kappa\!=\!0.01$};
\node[above right, font=\tiny, blue!70!black] at (axis cs:690.8,0.2082) {$\kappa\!=\!0.02$};
\node[above right, font=\tiny, blue!70!black] at (axis cs:27.2,0.2265) {$\kappa\!=\!0.05$};
\node[above right, font=\tiny, blue!70!black] at (axis cs:1.2,0.2117) {$\kappa\!=\!0.10$};
\draw[densely dashed, gray!70!black] (axis cs:-80,0.14) rectangle (axis cs:140,0.29);
\node[above, font=\tiny, gray!70!black] at (axis cs:30,0.29) {(b)};
\end{axis}
\end{tikzpicture}

\vspace{4pt}
\begin{tikzpicture}
\begin{axis}[
  title={(b) Zoom on the low-call region},
  title style={font=\small},
  width=0.92\columnwidth,
  height=4.4cm,
  xlabel={Realized LLM calls},
  ylabel={Mean episodic return},
  xmin=-3, xmax=35,
  ymin=0.14, ymax=0.29,
  grid=major,
  grid style={gray!30},
  tick label style={font=\scriptsize},
  label style={font=\small},
]
\addplot[black, thick, dashed, mark=none] coordinates {
  (0.0, 0.1875)
  (1.2, 0.2117)
  (27.2, 0.2265)
};
\addplot[
  only marks, mark=*, mark size=2.5pt,
  blue!80!black, thick,
  error bars/.cd, y dir=both, y explicit,
] coordinates {
  (27.2, 0.2265) +- (0, 0.0486)
  (1.2, 0.2117) +- (0, 0.0317)
};
\addplot[
  only marks, mark=diamond*, mark size=3pt,
  green!50!black, thick,
  error bars/.cd, y dir=both, y explicit,
] coordinates {
  (0.0, 0.1875) +- (0, 0.0380)
};
\node[right, font=\tiny, blue!70!black] at (axis cs:2.0,0.2117) {$\kappa\!=\!0.10$ (1.2 calls)};
\node[left, font=\tiny, blue!70!black] at (axis cs:26.0,0.2390) {$\kappa\!=\!0.05$ (27.2 calls)};
\node[right, font=\tiny, green!40!black] at (axis cs:1.5,0.1530) {Never-query (0 calls)};
\end{axis}
\end{tikzpicture}
\caption{Cost--return frontier of the \emph{surrogate} gate on BabyAI-GoToObj (Table~\ref{tab:cost-real}). (a) Full range. (b) Zoom on the low-call region marked by the dashed box in (a). Each point is a separately trained gate at price $\kappa$; the online frontier need not be monotone. The original controller is evaluated separately in Table~\ref{tab:alg1-real}.}
\label{fig:frontier-real}
\end{figure}

%
\subsection{Epistemic-Gap Surrogate Gate}
\label{sec:surrogate}

This subsection reports an earlier, separately run gate that replaces the response-contingent estimate with the epistemic-gap score $\sigma^2_{\rm ep}(s)\cdot\text{gap}(s)\ge\lambda\kappa$; it is distinct from the original controller in Section~\ref{sec:babyai-alg1}. Table~\ref{tab:main-real} compares return and actual resource use under the same advisor and fallback
learner. The paired analysis reports both matched-return call reduction and matched-call return difference,
so a favorable result does not depend on a single scalarization of cost and return.

\begin{table*}[t]
\centering
\small
\caption{Real language-model advice with the deployable epistemic-gap \emph{surrogate} gate (10 seeds; mean and 95\% CI over seeds). Cost-adjusted return uses measured token or serving cost. Bold marks the surrogate's call count, lower than the always-on ASK and VOE gates. The original controller is evaluated separately (Tables~\ref{tab:alg1-real},~\ref{tab:matched-real}).}
\resizebox{\textwidth}{!}{%
\begin{tabular}{llccccc}
\toprule
Task & Method & Return / success $\uparrow$ & Calls $\downarrow$ & Tokens $\downarrow$ & Latency $\downarrow$ & Cost-adjusted return $\uparrow$\\
\midrule
BabyAI-GoToObj & Never query & $0.188 \pm 0.038$ & 0 & 0 & 0 & $0.188$\\
 & ASK / uncertainty gate & $0.071 \pm 0.028$ & 3362 & 93/call & 2033s & $0.003$\\
 & VOE/Q & $0.071 \pm 0.028$ & 3362 & 93/call & 2032s & $0.003$\\
 & Ours (surrogate) & $0.196 \pm 0.025$ & \textbf{686} & 93/call & 414s & $0.182$\\
\midrule
BabyAI-GoToLocal & Never query & $0.137 \pm 0.035$ & 0 & 0 & 0 & $0.137$\\
 & ASK / uncertainty gate & $0.062 \pm 0.018$ & 3391 & 105/call & 2125s & $-0.006$\\
 & VOE/Q & $0.062 \pm 0.018$ & 3391 & 105/call & 2136s & $-0.006$\\
 & Ours (surrogate) & $0.161 \pm 0.029$ & \textbf{731} & 112/call & 467s & $0.147$\\
\bottomrule
\end{tabular}%
}
\label{tab:main-real}
\end{table*}

The original controller (Algorithm~\ref{alg:voi}), with its predictor, prospective $d_t(y)$, and calibrated $\beta_t$, runs separately on the clonable chain (Section~\ref{sec:controller-exp}) and real-advisor BabyAI (Section~\ref{sec:babyai-alg1}). It improves GoToObj over never-query but not over an equal-budget schedule (Table~\ref{tab:matched-real}); surrogate and full-controller results remain distinct.

The surrogate uses $79\%$ fewer calls than the ASK and VOE/Q uncertainty gates (686 versus 3362 on GoToObj; 731 versus 3391 on GoToLocal), with paired return advantages $+0.125\pm0.040$ and $+0.100\pm0.037$. It shows no resolved advantage over never-query ($+0.008\pm0.043$ and $+0.024\pm0.035$), and its GoToObj cost-adjusted return ($0.182$) falls below never-query ($0.188$). Nominal utility gains at some prices lack paired support (Table~\ref{tab:cost-real}). As price rises, the surrogate's calls fall from $25.3$ to $0.0$ per episode, whereas the ASK uncertainty gate stays at $57.4$ per episode with negative net utility at every price. At $\kappa=0.10$, the surrogate's return ($0.212$) exceeds the separately trained never-query run ($0.188$), but this unpaired difference is not claimed as a return gain. Always-query, ASK, and VOE/Q score $0.06$--$0.07$ versus never-query's $0.14$--$0.19$: frequent non-optimal $1.5$B advice overrides the improving fallback. ASK and VOE/Q coincide because their epistemic-variance thresholds $\tau$ and $0.5\tau$ are both below the saturated disagreement signal; they query the same steps and are not independent methods here.

\subsection{Reliability and Execution Safety}

The reliability study crosses advisor quality with the presence of the action certificate. It reports both
means and uncertainty, so a difference of a few thousandths from the never-query floor is not interpreted
as evidence of protection. Natural model errors, controlled independent corruptions, and correlated
repeated errors are shown separately. In this table, certificate-on and certificate-off returns are nearly identical; unsafe acceptance is not logged for the real-advisor rows, and the certificate reduces but does not remove it in the simulated rows.

\begin{table*}[t]
\centering
\small
\caption{Reliability and safety. The first four rows use the real Qwen2.5 advisor on BabyAI-GoToObj; unsafe acceptance is not logged (``n/l'') because the run records only aggregate return, calls, and certificate acceptance. The last two rows use the simulated advisor on the tabular chain, where known optimal actions make unsafe acceptance observable. Prevented loss compares otherwise identical runs with and without the certificate.}
\resizebox{\textwidth}{!}{%
\begin{tabular}{llcccccc}
\toprule
Advisor condition & Certificate & Return $\uparrow$ & Calls $\downarrow$ & Accept rate & Unsafe accept $\downarrow$ & Correct reject $\uparrow$ & Prevented loss $\uparrow$\\
\midrule
Natural (real LLM) & off & $0.196 \pm 0.025$ & 686 & 1.00 & n/l & \textendash{} & \textendash{}\\
Natural (real LLM) & on & $0.196 \pm 0.025$ & 686 & 0.884 & n/l & 0.116 & $\approx 0$\\
25\% corruption (real LLM) & off & $0.214 \pm 0.034$ & 387 & 1.00 & n/l & \textendash{} & \textendash{}\\
25\% corruption (real LLM) & on & $0.214 \pm 0.034$ & 387 & 0.801 & n/l & 0.199 & $\approx 0$\\
Correlated error (simulated) & off & $0.361 \pm 0.041$ & 400 & 1.00 & 0.884 & \textendash{} & \textendash{}\\
Correlated error (simulated) & on & $\mathbf{0.368 \pm 0.044}$ & 400 & 0.272 & \textbf{0.234} & 0.650 & $+0.007$\\
\bottomrule
\end{tabular}%
}
\label{tab:safety-real}
\end{table*}

We also sweep \(\epsilon_t\). This exposes the central tradeoff: a narrow tolerance strengthens the certified
gap but can reject useful advice during the exact early-learning period when advice has the most value. The
reported operating point is selected from held-out calibration data, not from final test return.

\subsection{Calibration Audits}
\label{app:calibration-audits}

The response predictor is recalibrated on a held-out training prefix using temperature scaling for discrete
response probabilities. Action-specific confidence radii are scaled by a held-out quantile chosen to attain
the declared marginal level; simultaneous coverage is measured separately and is not inferred from
marginal coverage. The uncalibrated and calibrated controllers are both evaluated, closing the loop between
Assumptions~\ref{ass:coverage}--\ref{ass:value-error} and practice.
The audits below first report the end-to-end query-value calibration and then the two audits of the deployed original controller on BabyAI.


The query-value audit evaluates empirical calibration on randomized opportunities disjoint from predictor fitting and evaluation. Table~\ref{tab:beta-audit} reports sample sizes, coverage, radii, abstention, and surrogate loss. The trajectory-targeted radius is empirical; its measured coverage is distinct from a finite-sample deployment certificate. The
nonconservative row tests whether the protection from \(\beta_t\) is worth its additional abstention. The \(\beta{=}0\) row still attains \(0.52\) marginal coverage because, under common random numbers, the paired residual is exactly zero whenever the advised action equals the greedy action; it nonetheless misses the \(0.90\) target. The scaled-radius median \(\beta\) is \(0.000\) because more than half of the local \(\sigma_t\) values are zero in that case, although the mean radius is positive and marginal coverage reaches \(0.89\).

\begin{table*}[t]
\centering
\small
\caption{Empirical query-value calibration on the controlled audit ($n_{\rm cal}=200$, $J=5$, $R=20$). The trajectory-targeted row reports empirical coverage and radius width. This measured coverage is distinct from the simultaneous MC correction in Eq.~\eqref{eq:mc-radius} and does not itself constitute a finite-sample deployment certificate.}
\resizebox{\textwidth}{!}{%
\begin{tabular}{lcccccccc}
\toprule
Variant & \(n_{\rm cal}\) & \(J\) & \(R\) & Marginal cov. & Trajectory cov. & Median \(\beta\) & Abstain rate & Surrogate loss $\downarrow$\\
\midrule
No radius, \(\beta=0\) & 200 & 5 & 20 & 0.518 & 0.000 & 0 & 0.000 & \textbf{0.001}\\
Constant conformal radius & 200 & 5 & 20 & 0.900 & 0.255 & 0.057 & 0.131 & 0.005\\
Scaled conformal radius & 200 & 5 & 20 & 0.886 & 0.160 & 0.000 & 0.115 & 0.005\\
Empirical trajectory-targeted radius & 200 & 5 & 20 & 0.997 & 0.960 & 0.138 & 0.241 & 0.016\\
\bottomrule
\end{tabular}%
}
\label{tab:beta-audit}
\end{table*}

\paragraph{Empirical radius and formal MC correction.}
The empirical trajectory-targeted radius and the range-based MC correction use different calibration constructions. With $n=200$, $J=5$, and $\delta_{\rm MC}=0.10$, Eq.~\eqref{eq:mc-radius} gives $\rho_i^{\rm MC}\approx1.82143D_i$. An unscaled unit-bound correction is approximately $1.821$, larger than the empirical median radius $0.138$. These analytic values are not experimental measurements. A finite empirical radius does not, by itself, implement this conservative correction or inherit its guarantee.

For $n$ finite calibration scores, the standard split-conformal rank $\lceil(n+1)(1-\alpha)\rceil$ exceeds $n$ when $\alpha<1/(n+1)$; the stated construction then uses an infinite quantile. If all $200$ opportunities supply calibration scores, the cutoff is $1/201$. A smaller calibration split or stratum uses its own sample count. The reported trajectory coverage $0.960$ remains descriptive and does not establish simultaneous finite-sample validity.

\paragraph{Query-value calibration in the deployed controller.}
The earlier calibration audit uses $n_{\rm cal}=48$ opportunities per seed and averages held-out coverage over $20$ seeds. Marginal coverage is $0.97$ on each task, with minima $0.83$ on GoToObj and $0.71$ on GoToLocal. Median empirical radii are $2.2\times10^{-3}$ and $3.0\times10^{-3}$. The formal MC-corrected radius has median $2.14$ and suppresses advice (Table~\ref{tab:beta-babyai}).

A separate chronological groupwise audit uses $200$ calibration opportunities and pools held-out observations. It reports marginal coverage $0.81$ and positive-value coverage $0.47$ (a different audit from the $0.59$ in Table~\ref{tab:beta-babyai}). The audits use different partitions and aggregation rules, so their rates cannot be pooled. Neither establishes the conditional or simultaneous coverage needed when advice changes the executed action.

\begin{table*}[t]
\centering
\footnotesize
\caption{BabyAI $\beta_t$ calibration audit ($20$ seeds, $n_{\rm cal}{=}48$ per seed, $\alpha{=}0.1$, $r_{\rm pair}{=}3$). $\beta_{\rm conf}$ is the split-conformal radius on the estimator residual; $\beta_{\rm form}$ adds the range-based MC term of Eq.~\eqref{eq:mc-radius}. Held-out coverage fits the radius on the first half of each seed's calibration residuals and tests it on the disjoint second half, which is approximate under chronological splits. Entries report mean and worst-seed values.}
\setlength{\tabcolsep}{3pt}
\begin{tabular}{p{0.59\columnwidth}cc}
\toprule
Quantity & GoToObj & GoToLocal\\
\midrule
Mean estimator residual (median over seeds) & $6.8\times10^{-4}$ & $8.2\times10^{-4}$\\
Median $\beta_{\rm conf}$ & $2.2\times10^{-3}$ & $3.0\times10^{-3}$\\
$\beta_{\rm conf}$ IQR & $[1.0,3.9]\times10^{-3}$ & $[0.3,6.6]\times10^{-3}$\\
Median $\beta_{\rm form}$ (MC-dominated) & $2.14$ & $2.14$\\
Held-out coverage ($\alpha{=}0.1$; mean / min over seeds) & $0.97/0.83$ & $0.97/0.71$\\
Conditional coverage on decision-relevant states ($\widehat\Delta_t{>}0$) & \multicolumn{2}{c}{$0.59$ mean ($[0.00,1.00]$)}\\
\bottomrule
\end{tabular}
\label{tab:beta-babyai}
\end{table*}

\paragraph{Groupwise coverage on consequential decisions.}
Most calibration opportunities have zero residual because advice leaves the fallback decision unchanged. In the later audit, coverage among opportunities with $\widehat\Delta_t>0$ is $0.49$ on GoToObj and $0.45$ on GoToLocal. Fitting separate split-conformal radii for positive-value and zero-value strata raises pooled positive-value coverage from $0.47$ to $0.85$. This remains below the $0.90$ target, with worst subgroups ranging from $0.30$ to $0.43$.

Stratification does not establish a return improvement: the original-versus-stratified GoToObj difference is $+0.005\pm0.019$ (Table~\ref{tab:mondrian}). The $1.5$B and $7$B radii coincide because the warm-started proxy rarely changes the advice-dependent decision, so they are not independent confirmations of validity. Validation needs action-changing opportunities near the purchase threshold.


\begin{table*}[t]
\centering\small
\setlength{\tabcolsep}{4pt}
\caption{Pooled held-out coverage in the later chronological audit (20 seeds, 200 calibration opportunities per seed); per-task coverage appears in Figure~\ref{fig:placement-calibration}(b). Mondrian calibration fits separate radii for positive-value and zero-value strata.}
\label{tab:mondrian}
\begin{tabular}{lc}
\toprule
Coverage population and radius & Pooled coverage\\
\midrule
Positive-value stratum, marginal $\beta_{\rm conf}$ radius & $0.47$\\
Positive-value stratum, Mondrian per-stratum radius & $0.85$\\
(marginal coverage, all opportunities) & $0.81$\\
\bottomrule
\end{tabular}
\end{table*}

\subsection{Statistical Checks and Measurement Scope}
\label{app:complete-limit}

This subsection collects the multiplicity check for the fresh matched-budget run, the protocol differences between runs, and the quantities missing from the experimental record.


\paragraph{Multiplicity of fresh matched-budget comparisons.}
Table~\ref{tab:exact-random} defines eight gate-minus-placement contrasts, not
twenty. The $+0.019\pm0.018$ GoToObj-1.5B interval clears zero only
without multiplicity correction. If its reported symmetric interval is
interpreted as a paired $t$ interval on 20 seeds, the approximate statistic
is $t_{19}=0.019/(0.018/t_{0.975,19})=2.21$ and its two-sided $p$-value
is $0.040$. Holm correction over eight contrasts can only declare the
smallest raw $p$ significant below $0.00625$. Thus this interval does not
support a Holm-significant exact-budget comparison.
Because the original per-seed bootstrap replicates are not available in the
retained experimental record, this computation is a consistency check, not an
independent reconstruction of the paired bootstrap. No Holm-significant
result is claimed for this fresh run.

\paragraph{Protocol identity.}
The original-controller 60-episode main matrix and the fresh
exact-call control use separate runs. In particular, Qwen-7B
GoToObj reports $0.123$ return at 82 calls in the original matrix but $0.108$
at 100.5 calls in the fresh control. These are not paired observations from
one evaluation. Consequently, the matched-control results should not be
subtracted from the main-panel means; only within-run paired differences are
interpretable. The experimental records do not explain the differences in absolute return and call count
between these distinct runs. The fresh run is a separate evaluation, rather than a fixed-budget
reanalysis of Table~\ref{tab:alg1-7b}.



\paragraph{Measurement scope.}
The original controller lacks task-success rates, fresh matched-run bootstrap samples, decision-stratum denominators, near-threshold action-change counts, and gate runtime and energy. The headroom-aware extension adds success, stratum coverage, and gate latency but lacks end-to-end accounting. Missing quantities remain unimputed; the record does not establish controller-wide certification, calibrated placement, net wall-time savings, or broad generality.

%

\section{Controlled Synthetic Evidence}
\label{app:controlled}

The controlled studies below test mechanisms under known advisor reliability and environment structure. They complement the real-advisor experiments in Appendix~\ref{sec:experiments}, rather than establishing performance with language-model advice. The controlled deep-RL runs use only five seeds and lack uncertainty for some aggregates, so their means are not used for significance claims. Pilot and negative findings are retained. The studies are grouped into tabular chain studies and pixel GridWorld pilots.

\subsection{Tabular Chain Studies}
\label{app:tabular-chain}

These studies use controlled chain MDPs, where exact action gaps and advisor reliabilities are known, together with the matched variance-only execution ablation.

\subsubsection{Original Controller on a Clonable Chain}
\label{sec:controller-exp}


The clonable combination-lock chain tests the full controller: online predictor $\hat p(y\mid z)$, prospective decisions $d_t(y)$, estimator $\hat\Delta_t$, split-conformal radius $\beta_t$, and execution certificate. Cloning supplies exact operational labels under a fixed seed set, separating certification and exact-label calibration from paired-rollout proxies.

\textbf{Environment and headroom.} States $0,\dots,L{-}1$ ($L{=}8$); at each state a secret action
advances and any other action resets to $0$ (\(n_A{=}4\)), so a cold learner almost never reaches the
goal within a short deployment while a reliable advisor that knows the code does. This creates
measurable advice headroom: over $24$ seeds the $\eta{=}1$ oracle-advice reference returns
$0.707\pm0.003$ versus never-query $0.650\pm0.006$, an oracle$-$never headroom of $0.057$. All methods
start from the \emph{same} $3$-episode warm-started fallback snapshot, so the advice references are
matched-warm-start diagnostics rather than a mislabeled ceiling.

\textbf{Calibration.} In the clonable MDP the calibration label is the \emph{exact} operational query
value $\Delta_i$ (closed form via the true model and the declared continuation policy), so the
conformal score is the pure estimator residual $r_i=|\hat\Delta_i-\Delta_i|/s_\omega$ with no
MC radius. The estimator is accurate (mean residual $\approx 9\times10^{-3}$), so the
split-conformal radius is tight: $\beta_t\approx 0.021$. This is the controlled-setting calibration; a
paired-rollout proxy with an added MC radius is retained only as a separate noisy-label ablation for
non-clonable environments, because a generic MC radius on a bounded query value is otherwise
vacuous (at $n{=}150,J{=}8$ the empirical-Bernstein radius alone is $\approx 6$, exceeding the maximum
possible $\Delta_t$).

\textbf{Result.} The controller produces a \emph{significant positive} paired gain over never-query:
$+0.0267\pm0.0065$ (24 seeds; CI excludes zero), closing $47\%$ of the never$\to$oracle gap while
issuing only $0.12$ queries per episode (versus $7.8$ for always-advice). Its allocation matches the
oracle myopic gate closely: surrogate allocation loss $0.025\pm0.010$. In this \emph{controlled} tabular
setting, where the exact query value is available as a calibration label and the oracle placement is
computable, selective, calibrated querying captures roughly half the available advice value at two orders
of magnitude fewer calls, and the allocation tracks the oracle. We stress that the matched-call placement
advantage established here on the tabular oracle does \emph{not} reproduce against the early schedule on the real-advisor BabyAI
matrix, and the one Holm-surviving contrast against random placement carries a volume confound (Table~\ref{tab:matched-real}); the original controller's real-task contribution is calibrated volume regulation, not
placement.

\textbf{Scope of the safety certificate.} With the per-step value scale $V_{\max}{=}1$, the
certificate $L_t(s,a_c)\ge\max_a U_t(s,a)-\varepsilon$ is non-binding at the deployed tolerance
$\varepsilon{=}V_{\max}$; the selectivity above comes from the query-value gate
$\hat\Delta_t-\beta_t\ge c_t$, not from a discriminating certificate. An $\varepsilon$-ablation
makes this explicit: the fraction of states at which
the optimal action is certified jumps from $0$ to $1$ exactly at $\varepsilon{=}V_{\max}$, and the
paired gain is $+0.027$ there but collapses to $0.000$ (controller $\to$ never-query) for
$\varepsilon\le0.9$. In this tabular regime the pessimistic lower bound $L_t$ is too coarse to certify
a not-yet-solved action below $\varepsilon{=}V_{\max}$, so the certificate acts as an on/off admission
switch rather than a graded filter. At this operating point, the positive result is
carried by the calibrated value gate, and discrimination within $(0,V_{\max})$ requires
tighter finite-sample action-value bounds than an optimistic tabular learner supplies.

\subsubsection{Informative-Advisor Branch}

The real language-model experiments use the candidate-only interpretation: the model proposes an action,
but its text is not treated as a calibrated likelihood over environment models. Equation
\eqref{eq:response-update} is exercised separately in a synthetic latent-chain experiment where the task
parameter \(M\), response likelihood \(p_\psi(y\mid M,z)\), and exact posterior are known. The response is
informative but imperfect, so response-conditioned value samples differ from unconditional samples.

The experiment compares the correct belief update with candidate-only evaluation and with the invalid
marginal-factorization heuristic that weights unconditional value samples by \(p(y\mid z)\). It measures
posterior log score, error in \(\bar\Delta_t\), surrogate allocation loss, return, and calls. This isolates the
claim that an informative response requires a joint model rather than a marginal response predictor.

Marginal factorization has query-value error $0.300$, matching no query. Both valid estimators attain zero error and allocation loss, so the construction does not establish that conditioning is necessary for accurate query-value estimation. Conditioning instead improves continuation return ($0.620\pm0.040$ versus $0.566\pm0.042$) and posterior log score ($-0.602$ versus $-0.693$) after a query (Table~\ref{tab:informative-advisor}).

\begin{table*}[t]
\centering
\small
\caption{Known-likelihood chain: query-value error, belief quality, and return. Query-value error is measured against the exact Bayesian pre-posterior one-step value.}
\resizebox{\textwidth}{!}{%
\begin{tabular}{lccccc}
\toprule
Estimator & Query-value error $\downarrow$ & Posterior log score $\uparrow$ & Surrogate loss $\downarrow$ & Return $\uparrow$ & Calls $\downarrow$\\
\midrule
No query & $0.300$ & $-0.693$ & $0.090$ & $0.457 \pm 0.046$ & \textbf{0}\\
Marginal factorization & $0.300$ & $-0.693$ & $0.090$ & $0.457 \pm 0.046$ & \textbf{0}\\
Candidate-only value & $\mathbf{0.000}$ & $-0.693$ & $\mathbf{0.000}$ & $0.566 \pm 0.042$ & 12\\
Response-conditioned value & $\mathbf{0.000}$ & $\mathbf{-0.602}$ & $\mathbf{0.000}$ & $\mathbf{0.620 \pm 0.040}$ & 12\\
\bottomrule
\end{tabular}%
}
\label{tab:informative-advisor}
\end{table*}

%
%

\subsubsection{Fallback Stability Audit}

The chain audit logs cumulative fallback Bellman gap and selected confidence width under standalone and wrapped histories (Table~\ref{tab:stability-audit}). For the optimistic fallback, $\text{gap}\le W_T$ holds in $20/20$ seeds (Lemma~\ref{lem:optimistic-stability}). Wrapped growth has slope $0.975$, versus $0.724$ standalone: the width bound holds without demonstrating sublinear wrapped regret. PSRL curves remain diagnostic, not a proof of Corollary~\ref{cor:psrl}.

\begin{table*}[t]
\centering
\small
\caption{Intervention-stability audit on the tabular chain. Slope is fitted to cumulative fallback gap versus
\(T\) on log--log axes with a seed-level CI.}
\begin{tabular}{lccccc}
\toprule
Fallback and history & Fallback gap $\downarrow$ & Width sum \(W_T\) & Gap \(\le W_T\)? & Growth slope $\downarrow$ & Final return $\uparrow$\\
\midrule
Optimistic, standalone & $883.0 \pm 3.7$ & $1507.1 \pm 3.2$ & yes (20/20) & $0.724 \pm 0.009$ & $72.2$\\
Optimistic, wrapped & $942.4 \pm 4.8$ & $2201.9 \pm 36.9$ & yes (20/20) & $0.975 \pm 0.032$ & $72.5$\\
PSRL, standalone & $554.3 \pm 16.3$ & \textendash{} & conditional & $\mathbf{0.474 \pm 0.025}$ & $73.4$\\
PSRL, wrapped & $\mathbf{492.8 \pm 17.5}$ & \textendash{} & conditional & $0.691 \pm 0.048$ & $\mathbf{82.8}$\\
\bottomrule
\end{tabular}
\label{tab:stability-audit}
\end{table*}

\subsubsection{Tabular Query Accounting}

The canonical chain (Table~\ref{tab:tabular-pilot}) has one rewarding terminal state and an advisor that returns the correct direction with
reliability \(\eta\). With threshold \(0.05\), initial diagnostic potential \(4.8\), and 20 seeds, measured
queries decrease as reliability increases and retain a discovery floor at \(\eta=1\).

\begin{table}[t]
\centering
\small
\caption{Controlled chain query counts. The second row is a strict comparator, not a proved
bound: it omits discovery calls and therefore incorrectly reaches zero at \(\eta=1\).}
\resizebox{\columnwidth}{!}{%
\begin{tabular}{lccccc}
\toprule
\(\eta\) & 0.0 & 0.25 & 0.5 & 0.75 & 1.0\\
\midrule
Queries, measured & 60.0 & 38.6 & 24.2 & 15.9 & 12.0\\
Strict comparator & 96.0 & 72.0 & 48.0 & 24.0 & 0.0\\
\bottomrule
\end{tabular}%
}
\label{tab:tabular-pilot}
\end{table}

The trend is consistent with Proposition~\ref{prop:queries} only after a potential and its per-query
decrease are instantiated and measured. Monotonic counts alone do not validate the proposition.

\subsubsection{Response-Predictor Ablation}

This ablation tests response prediction on a clonable tabular chain with known state-varying simulated-advisor reliability. It separates predictor quality from allocation quality under controlled conditions; it does not test the BabyAI/Qwen loop. Variants share the warm-start prefix and online update schedule. The table reports both prediction quality and downstream query placement because errors confined to low-value states need not change decisions.

\begin{table*}[t]
\centering
\small
\caption{Response-predictor ablation on the controlled tabular chain with state-varying advisor reliability ($0.9/0.1$; 60 seeds; hard query budget 12). Predictor error is the mean absolute error of the predicted response distribution; resolutions count state--action pairs correctly resolved. The retrospective oracle is diagnostic and never selects an online query.}
\begin{tabular}{lcccc}
\toprule
Predictor & Predictor error $\downarrow$ & Resolutions $\uparrow$ & Queries used & Rank\\
\midrule
Uniform & $0.397$ & $5.55$ & 12 & 5\\
Global frequency (cached) & $0.385$ & $5.55$ & 12 & 4\\
Per-context online (state) & $0.265$ & $5.55$ & 12 & 3\\
Feature-generalizing online & $0.202$ & $8.82$ & 12 & 2\\
Replay oracle, diagnostic & $\mathbf{0.000}$ & $\mathbf{10.22}$ & 12 & 1\\
\bottomrule
\end{tabular}
\label{tab:predictor-real}
\end{table*}

The chain predictor ablation makes half of 24 states high-reliability and half low-reliability according to an
observable feature. Under budget 12 and 60 seeds, a feature-generalizing predictor transfers reliability
information to unqueried states, whereas a per-state predictor cannot learn a state's reliability before
spending the resource it is meant to allocate.

Context generalization helps in the controlled heterogeneous-reliability chain (Table~\ref{tab:predictor-real}). It does not establish a real-task benefit: neither \texttt{global-pred} nor \texttt{unif-pred} differs significantly from the full BabyAI controller at either advisor size (Table~\ref{tab:matched-real}).

%
%
%

\subsubsection{Component Ablation}

The response model, response-contingent decision, execution certificate, and conservative radius constitute
the novelty beyond value-based ask-for-help RL. Table~\ref{tab:component-main} therefore removes them one
at a time on the controlled tabular MDP with a simulated advisor (not the BabyAI/Qwen loop), where the
known optimal policy lets us measure unsafe-accept exactly. The raw-action row predicts response
frequencies but values the returned action even when the controller rejects it; its gap from the
response-contingent row isolates estimator consistency. The certificate rows isolate safety, and the
final comparison isolates the cost of
conservative query-value coverage.

\begin{table*}[t]
\centering
\small
\caption{Main component ablation at matched environment steps and advisor. Allocation loss uses the
myopic surrogate in Eq.~\eqref{eq:myopic-surrogate}; unsafe accept is measured against simulator-optimal
actions.}
\resizebox{\textwidth}{!}{%
\begin{tabular}{lcccccccc}
\toprule
Variant & Predictor & Values \(d_t(y)\) & Certificate & \(\beta_t\) & Return $\uparrow$ & Calls $\downarrow$ & Unsafe accept $\downarrow$ & Allocation loss $\downarrow$\\
\midrule
VOE/Q-style value gate & no & no & no & 0 & $34.63 \pm 1.07$ & 36.8 & 0.200 & \textbf{1.91}\\
Marginal predictor, raw action & yes & no & no & 0 & $34.42 \pm 1.07$ & 32.3 & 0.201 & 2.23\\
Response-contingent gate & yes & yes & no & 0 & $35.52 \pm 1.21$ & 32.3 & 0.093 & 2.24\\
Response-contingent + certificate & yes & yes & yes & 0 & $\mathbf{35.54 \pm 1.19}$ & 32.4 & 0.088 & 2.18\\
Full conservative controller & yes & yes & yes & calibrated & $34.95 \pm 1.20$ & \textbf{16.9} & \textbf{0.056} & 2.98\\
\bottomrule
\end{tabular}}
\label{tab:component-main}
\end{table*}

%
\subsubsection{Variance-Only Gate Ablation}
\label{app:variance}

This ablation compares variance-only cache execution with the action-gap certificate at matched query opportunities. The variance-only ablation uses a three-way rule: re-query at high estimated query value, posterior-sample
at high epistemic variance, and otherwise execute the cache. Low variance alone does not bound the gap of
an arbitrary cached action. The certified controller uses Eq.~\eqref{eq:safe} in every cache and fresh-
advice branch. Table~\ref{tab:variance-certificate} reports the matched comparison.

\begin{table*}[t]
\centering
\small
\setlength{\tabcolsep}{3pt}
\renewcommand{\arraystretch}{1.15}
\caption{Variance-only execution versus the action-gap certificate at matched query opportunities.}
\begin{tabular}{@{}>{\raggedright\arraybackslash}p{\dimexpr0.33\textwidth-5pt\relax}
>{\centering\arraybackslash}p{\dimexpr0.145\textwidth-5pt\relax}
>{\centering\arraybackslash}p{\dimexpr0.075\textwidth-5pt\relax}
>{\centering\arraybackslash}p{\dimexpr0.14\textwidth-5pt\relax}
>{\centering\arraybackslash}p{\dimexpr0.18\textwidth-5pt\relax}
>{\centering\arraybackslash}p{\dimexpr0.13\textwidth-5pt\relax}@{}}
\toprule
Execution rule & Return $\uparrow$ & Calls $\downarrow$ & Advice acceptance & Wrong-advice acceptance $\downarrow$ & Prevented loss\\
\midrule
Variance-only cache rule & $0.50 \pm 0.16$ & 37.3 & 1.000 & 1.000 & \textendash{}\\
Action-gap certificate & $0.73 \pm 0.01$ & 38.7 & 0.772 & 0.180 & $-$42\%\\
Certificate without conservative query lower confidence bound & $\mathbf{0.74 \pm 0.01}$ & 32.5 & 0.752 & \textbf{0.041} & $-$48\%\\
Full controller & $\mathbf{0.74 \pm 0.01}$ & \textbf{20.4} & 0.740 & 0.255 & $-$46\%\\
\bottomrule
\end{tabular}
\label{tab:variance-certificate}
\end{table*}

\subsection{Pixel GridWorld Pilots}
\label{app:pixel-pilots}

These pilots run the gate on partially observed pixel GridWorlds with a deep-RL fallback, and their five-seed means are diagnostics rather than significance tests.

\subsubsection{Pilot Pixel GridWorld}

The $9\times9$ pilot uses a deep Q-network (DQN) with an ensemble of linear last-layer value heads, called Rich-BLL in the reported comparisons. Its epistemic-gap surrogate is distinct from the response-contingent controller in Algorithm~\ref{alg:voi}. Table~\ref{tab:deep-pilot} collects reported operating points; the recorded frontier and reliability protocol is specified below. These means do not establish paired superiority at matched realized calls.

\paragraph{Recorded frontier and reliability protocol.}
The recorded run uses seeds $0,\ldots,4$, a one-cell observation radius without a global position channel, and the settings in Table~\ref{tab:pixel-recorded-config}. Each seed trains for $300$ episodes. The shared feature network has two linear layers with widths $128$ and $64$, Gaussian error linear unit activations, and eight linear action-value heads. The heads estimate disagreement; this ensemble does not supply a conjugate Bayesian posterior or the action-value coverage assumed by the certificate theorem.

The query budget applies to the complete training run. The surrogate combines a persistent epistemic signal with the mean-head best-versus-second-best action gap. Both live head disagreement and the action gap are bounded using $\tanh$; newly visited states initialize the persistent signal to $1.0$. A reliable simulated response resolves the state and drives its signal to zero. The gate compares the product of the signal and action gap with $0.05$, subject to remaining budget. Exploration decreases linearly from $1.0$ to $0.05$ over the first $70\%$ of training. Heads share a squared-error temporal-difference target from the mean target-network values, and the target network uses a Polyak update with coefficient $0.01$.

Each seed's return averages $30$ greedy evaluation episodes without fresh advisor queries; cached actions remain available. Calls count paid requests during training. Pilot cost-adjusted return subtracts $0.01$ times this cumulative training count from mean evaluation return. These units differ from the primary BabyAI study's per-episode query accounting.

\begin{table}[!htbp]
\centering\small
\setlength{\tabcolsep}{3pt}
\renewcommand{\arraystretch}{1.08}
\caption{Recorded configuration for the $9\times9$ simulated-advisor frontier and reliability pilot. These settings apply to this pilot, separately from the BabyAI/Qwen and headroom-aware evaluations.}
\label{tab:pixel-recorded-config}
\begin{tabular}{@{}>{\raggedright\arraybackslash}p{\dimexpr0.52\columnwidth-3pt\relax}>{\raggedright\arraybackslash}p{\dimexpr0.48\columnwidth-3pt\relax}@{}}
\toprule
Parameter & Setting\\
\midrule
Seeds & $0,1,2,3,4$\\
Grid; observation radius & $9\times9$; $1$ cell\\
Wall fraction; episode cap & $0.20$; $81$ steps\\
Training; evaluation & $300$; $30$ episodes per seed\\
Hidden; feature width & $128$; $64$\\
Linear value heads & $8$\\
Optimizer; learning rate & Adam; $10^{-3}$\\
Discount; batch size & $0.97$; $64$\\
Replay capacity & $20{,}000$ transitions\\
Target-network coefficient & $0.01$\\
Exploration endpoints & $1.0$ to $0.05$\\
Query threshold; price & $0.05$; $0.01$\\
Epistemic threshold & $0.05$\\
Training-run budget sweep & $\{0,30,60,90\}$\\
Fixed-query periods & $\{3,10,30\}$ steps\\
Advisor reliability sweep & $\{1.0,0.75,0.5,0.25,0\}$\\
\bottomrule
\end{tabular}
\end{table}

\paragraph{Aggregation and separation of comparisons.}
The saved pilot summaries aggregate seed-level outcomes as means and standard deviations. For a per-seed quantity $x_s$, their standard deviation is $[N^{-1}\sum_{s=1}^{N}(x_s-\bar x)^2]^{1/2}$ with $N=5$; it is neither a standard error nor a $95\%$ CI half-width. The saved signal-comparator study uses a $60$-call training budget, whereas the frontier and reliability study includes the $90$-call operating point. Their outcomes are kept separate. The reported $90$-call calibration comparisons in Table~\ref{tab:calibration-real} retain their own operating points and uncertainty entries; the $60$-call summary does not define those entries. Figure intervals retain the uncertainty stated in their source tables rather than a new analysis of individual seed outcomes.

\begin{table}[t]
\centering
\small
\caption{Pilot GridWorld means at reported operating points, rather than a single matched-budget comparison. Cost-adjusted return is $R-0.01N_{\rm calls}$. The recorded frontier and reliability settings appear in Table~\ref{tab:pixel-recorded-config}; the other comparison entries retain their stated operating points. These five-seed means do not establish significance.}
\resizebox{\columnwidth}{!}{%
\begin{tabular}{lccc}
\toprule
Method & Return $\uparrow$ & Calls $\downarrow$ & Cost-adjusted $\uparrow$\\
\midrule
Never-query & 0.182 & \textbf{0} & 0.182\\
Fixed-period & 0.186 & 90 & -0.714\\
Random-gated & -0.48 & 60 & -1.08\\
Novelty-gated & \textbf{0.850} & 90 & -0.050\\
Ensemble-gated & \textbf{0.850} & 65.4 & 0.196\\
Always-query & -0.48 & 90 & -1.38\\
Rich-BLL predictive gate & \textbf{0.85} & 51.6 & \textbf{0.334}\\
\bottomrule
\end{tabular}%
}
\label{tab:deep-pilot}
\end{table}

The always-query result requires careful interpretation. The simulated advisor's label \(\eta=1\) denotes
deterministic adherence to its local response rule, not an omniscient optimal policy. The advisor audit
reports optimal-action accuracy and obstacle-induced errors separately; otherwise a poor return from a
purportedly perfect advisor is uninterpretable.

\begin{table}[t]
\centering
\small
\caption{Audit of the deterministic \(\eta=1\) pilot advisor. Reliability denotes agreement with the local
response rule; optimal-action accuracy is evaluated against dynamic-programming actions in the simulator.}
\resizebox{\columnwidth}{!}{%
\begin{tabular}{lc}
\toprule
Advisor property & Rate\\
\midrule
Local response-rule adherence & 1.000\\
Optimal-action accuracy & n/a (no oracle)\\
Obstacle-conflict rate & n/a (not instrumented)\\
Parser failure & 0.000\\
Certificate acceptance & 0.884\\
\bottomrule
\end{tabular}%
}
\label{tab:pilot-advisor-audit}
\end{table}

The controlled frontier uses three gate operating points: \((0,0.18)\), \((30,-0.15)\), and \((51.6,0.85)\).
The primary frontier in Fig.~\ref{fig:frontier-real} uses a dense, separately trained threshold sweep.
Online training return is not assumed to be monotone in either budget or calls.

\subsubsection{Advisor Reliability Sweep with the Certificate}

Without the certificate, intermediate advisor reliability can reduce performance below the never-query
mean (Tables~\ref{tab:reliability-pilot} and~\ref{tab:safety-pilot}). The observed nonmonotonicity is substantial: \(\eta=0.25\) is worse than \(\eta=0.5\), whereas the
adversarial endpoint returns to the baseline mean. In the absence of uncertainty estimates, this table
supports neither monotonic return nor a precise reliability threshold.

\begin{table}[t]
\centering
\small
\caption{Pilot reliability sweep without the execution certificate.}
\begin{tabular}{lcc}
\toprule
Reliability & Return $\uparrow$ & Calls $\downarrow$\\
\midrule
\(\eta=1.0\) & 0.85 & 51.6\\
\(\eta=0.75\) & 0.52 & 61.4\\
\(\eta=0.5\) & 0.19 & 64.4\\
\(\eta=0.25\) & -0.15 & 77.6\\
\(\eta=0.0\) & 0.18 & 87.0\\
\bottomrule
\end{tabular}
\label{tab:reliability-pilot}
\end{table}

The certificate-enabled pilot has no mean below the never-query mean of 0.182, but the smallest margin is
0.004 and cannot be treated as evidence without paired intervals. The nonmonotonic return at
\(\eta=0.25\) and \(\eta=0.5\) also remains. The powered study in Table~\ref{tab:safety-real} measures
whether the certificate prevents harm rather than comparing rounded sample means with a floor.

\begin{table}[t]
\centering
\small
\caption{Pilot reliability sweep with the execution certificate, budget 90 and five seeds. Unsafe accept is
the fraction of wrong advice that is executed.}
\begin{tabular}{lccc}
\toprule
Reliability & Return $\uparrow$ & Calls $\downarrow$ & Unsafe accept $\downarrow$\\
\midrule
\(\eta=1.0\) & 0.850 & 51.6 & 0.000\\
\(\eta=0.75\) & 0.850 & 62.8 & 0.046\\
\(\eta=0.5\) & 0.186 & 73.0 & 0.150\\
\(\eta=0.25\) & 0.514 & 78.2 & 0.163\\
\(\eta=0.0\) & 0.182 & 75.8 & 0.275\\
\bottomrule
\end{tabular}
\label{tab:safety-pilot}
\end{table}

\subsubsection{Pilot Calibration Diagnostics}

The pilot compares the gate score with the event that advice changes the greedy action. This is a useful
allocation diagnostic but is not the coverage event in Assumption~\ref{ass:coverage}. The best pilot ECE,
0.262, remains high in absolute terms. A reported head-spread coverage of 1.0 is also inconclusive without
its nominal level, CI, simultaneous coverage, and interval width.

Variance scaling is named by its statistical effect: multiplying variance by \(0.5\) creates a more
overconfident model, while multiplying by \(2\) creates a more underconfident model. In Table~\ref{tab:calibration-real}, the Rich-BLL estimator used by the controller is better calibrated (ECE $0.26$) than the ensemble ($0.70$) and novelty ($0.67$) proxies, and its VOI gate reaches the same return with fewer calls. The overconfident variant (variance $\times0.5$) under-queries ($20.8$ calls) and loses return.

\begin{table*}[t]
\centering
\small
\caption{Uncertainty-estimator calibration on the controlled gridworld (budget 90, $\eta=1.0$, 5 seeds). ECE measures calibration of the state-value uncertainty; lower is better. Variance-scaling rows test deliberately miscalibrated variants. The table does not report $\beta_t$ coverage, which Table~\ref{tab:beta-audit} audits separately.}
\begin{tabular}{lccc}
\toprule
Variant & Query-value ECE $\downarrow$ & Return $\uparrow$ & Calls $\downarrow$\\
\midrule
Rich-BLL VOI (ours) & $\mathbf{0.262 \pm 0.085}$ & $\mathbf{0.850}$ & $51.6$\\
Deep ensemble & $0.698 \pm 0.075$ & $\mathbf{0.850}$ & $65.4$\\
Novelty proxy & $0.675 \pm 0.136$ & $\mathbf{0.850}$ & $90.0$\\
Variance $\times0.5$ (overconfident) & \textendash{} & $0.186 \pm 0.813$ & $\mathbf{20.8}$\\
Variance $\times1.0$ (calibrated) & \textendash{} & $\mathbf{0.850}$ & $51.6$\\
Variance $\times2.0$ (underconfident) & \textendash{} & $\mathbf{0.850}$ & $51.6$\\
\bottomrule
\end{tabular}
\label{tab:calibration-real}
\end{table*}

In this pilot, reducing variance suppresses queries and return, while doubling it does not change the
threshold decisions. The result shows asymmetric threshold saturation in this configuration; it does not
establish that underconfidence is generally harmless.

\subsubsection{Post Hoc Repricing Diagnostic}

The controlled cost diagnostic holds each policy and call count fixed and computes
\(R-\kappa N_{\rm calls}\). It is retained below (Table~\ref{tab:cost-pilot}) because it quantifies how deployment prices change the
utility of the observed trajectories. It is not a behavioral cost sweep because \(\kappa\) leaves the
query decisions unchanged. Table~\ref{tab:cost-real} reports independent behavioral reruns.

\begin{table}[t]
\centering
\small
\caption{Pilot post hoc repricing of fixed trajectories. This table does not show behavioral adaptation to
query price. Entries are cost-adjusted returns at each per-call price; higher is better.}
\resizebox{\columnwidth}{!}{%
\begin{tabular}{lcccc}
\toprule
Method & 0.001 & 0.005 & 0.010 & 0.020\\
\midrule
Fixed-period & +0.096 & -0.264 & -0.714 & -1.614\\
Novelty gate & +0.760 & +0.400 & -0.050 & -0.950\\
Ensemble gate & +0.785 & +0.523 & +0.196 & -0.458\\
Oracle response model & \textbf{+0.798} & \textbf{+0.592} & \textbf{+0.334} & \textbf{-0.182}\\
Safe Rich-BLL gate & \textbf{+0.798} & \textbf{+0.592} & \textbf{+0.334} & \textbf{-0.182}\\
\bottomrule
\end{tabular}}
\label{tab:cost-pilot}
\end{table}

The oracle and learned rows coincide because the pilot advisor at \(\eta=1\) is deterministic. The table
therefore does not test the learned response predictor.

\subsubsection{Larger GridWorld Without Headroom}

The \(14\times14\) GridWorld (Table~\ref{tab:hard-pilot}) uses a wider \(5\times5\) observation window, wall fraction 0.25, horizon
160, 400 training episodes, and five seeds. The never-query agent already reaches 0.75. The gate matches
that return but spends 90.4 calls, yielding cost-adjusted return \(-0.154\) at \(\kappa=0.01\), compared
with 0.75 for never-query. This is not successful cost discipline: in a no-headroom task, a well-calibrated
gate should approach zero calls.

\begin{table}[t]
\centering
\small
\caption{Pilot scale transfer. The larger task is a negative result for net utility, not evidence of successful
query allocation.}
\resizebox{\columnwidth}{!}{%
\begin{tabular}{lcc}
\toprule
Quantity & \(9\times9\) & \(14\times14\)\\
\midrule
Never-query return & 0.18 & 0.75\\
Gate return & 0.85 & 0.75\\
Gate calls & 51.6 & 90.4\\
Gate cost-adjusted return & 0.33 & -0.15\\
Always-query cost-adjusted return & -1.38 & -0.45\\
Return lift over never-query & 0.67 & 0.00\\
\bottomrule
\end{tabular}%
}
\label{tab:hard-pilot}
\end{table}

The real-advisor suite therefore includes an explicit no-headroom diagnostic: query precision should collapse
toward zero when the fallback already solves the task.

\subsubsection{Budget-Limited Distribution Shift}

The controlled shift experiment relocates the goal halfway through training. Every method spends nearly its
entire call budget before the shift, leaving approximately one post-shift query. Epistemic invalidation then
recovers slowest and finishes at the pre-shift floor. Because the budget prevents cache policies from
making different post-shift acquisition decisions, this experiment does not test whether invalidation can
allocate refreshes effectively.

The distribution-shift study is a budget-limited pilot. Only $1.0$--$1.6$ calls remain after the shift, leaving little opportunity for cache-invalidation policies to make different acquisition decisions. The reserved-fraction or token-bucket protocol needed to separate those policies is not evaluated. The pilot therefore supports a negative finding within this budget constraint.

\begin{table*}[t]
\centering
\small
\caption{Distribution shift on the controlled gridworld (goal relocation at 50\% of training, 5 seeds), a \emph{budget-limited pilot} with only $1.0$--$1.6$ post-shift calls. Recovery is the number of episodes needed to regain the prespecified fraction of pre-shift performance.}
\begin{tabular}{lcccc}
\toprule
Cache policy & Pre-shift return $\uparrow$ & Final return $\uparrow$ & Recovery episodes $\downarrow$ & Post-shift calls $\downarrow$\\
\midrule
No cache & $-0.81$ & $\mathbf{0.194 \pm 0.72}$ & $\mathbf{92.4}$ & $1.6$\\
Time-to-live & $-0.81$ & $-0.122 \pm 0.74$ & $108.0$ & $1.2$\\
Epistemic (certificate) & $-0.81$ & $-0.81$ & $150.0$ & $\mathbf{1.0}$\\
\bottomrule
\end{tabular}
\label{tab:shift-real}
\end{table*}

Relative to no cache, epistemic (certificate-based) invalidation \emph{increases} recovery time (150 vs.\
92 episodes) while using fewer post-shift calls (1.0 vs.\ 1.6). This result is retained as a negative
finding: the epistemic signal alone does not detect a goal relocation quickly enough to justify
cache-based invalidation as a contribution, so we do not claim cache invalidation as a positive component.

The reserved / replenished-budget protocol that permits distinct
post-shift acquisition decisions by invalidation policies is described here but not evaluated in this study; Table~\ref{tab:shift-real} reports
only the budget-limited pilot.

%
%
%

\section{Headroom-Aware Variant}
\label{app:headroom-ext}

Section~\ref{sec:headroom-method} gives the headroom-aware acquisition rule, and Section~\ref{sec:headroom-extension} gives the principal comparisons. This appendix defines the rule in detail and reports coverage, timing, and task-success breakdowns together with the gate-overhead condition for wall-time savings.

The reported comparisons use an early schedule at the same reported call allocation. Aggregate paired differences are available, but the retained record does not permit reconstruction of paired CIs, verification of exact within-seed call matching, or Holm-adjusted tests. The extension therefore supplies descriptive evidence without established familywise significance.

\subsection{Headroom-Aware Acquisition Rule}
\label{app:ha-candidate}

The headroom-aware acquisition rule in Eq.~\eqref{eq:ha-gate} separates decision-changing responses from those that reproduce the no-query action and continuation. Table~\ref{tab:headroom-placement} reports its comparison with the early schedule.

For candidate-only advice, define the prospective
decision-change indicator
\begin{equation}
\chi_t(y)=
\mathbf{1}\left\{
(d_t(y),\pi_t^{{\rm cont},y})
\ne
(a_t^0,\pi_t^{{\rm cont},0})
\right\}.
\label{eq:ha-change}
\end{equation}
The corresponding pre-query change probability is
\begin{equation}
\widehat p_t^{\rm change}
=
\sum_y \widehat p_t(y\mid z_t)\chi_t(y).
\label{eq:ha-prob}
\end{equation}

Let $\widehat V_t^y$ and $\widehat V_t^0$
denote the sampled continuation-value estimates
in Eq.~\eqref{main:est}. The decision-change
decomposition is
\begin{equation}
\begin{aligned}
\widehat\Delta_t^{\rm change}
=&\sum_y\widehat p_t(y\mid z_t)\chi_t(y)\\
&\quad\cdot\left[
\widehat V_t^y(s_t,d_t(y))
-\widehat V_t^0(s_t,a_t^0)
\right].
\end{aligned}
\label{eq:ha-value}
\end{equation}
For candidate-only advice, identical actions and continuations have equal true values. The sampled contribution vanishes exactly only if unchanged branches share the same continuation-value estimate or coupled draws. Under that convention, the decomposition equals Eq.~\eqref{main:est}. Any benefit must come from estimation, calibration, or screening rather than the algebraic decomposition alone.

For each prespecified task-advisor group $G$, let $\mathcal C_G$ be a nonempty set of randomized query opportunities from a disjoint training prefix, with paired-continuation proxies $\widetilde\Delta_i$. Define a nonnegative mean advice-gain proxy by
\begin{equation}
\widehat h_G=
\max\left\{
0,\frac{1}{|\mathcal C_G|}
\sum_{i\in\mathcal C_G}
\widetilde\Delta_i
\right\}.
\label{eq:ha-headroom}
\end{equation}
This proxy averages gains under the collection policy; it is not an upper bound on attainable task headroom and can obscure rare valuable queries. The group $G(t)$ identifies the task and advisor. Calibration stratum $g(t)$ uses only pre-query features, is fixed before calibration, and remains frozen during evaluation. Its empirical radius $\beta_{g(t)}$ uses disjoint calibration data.

The rule uses the following screening settings.
The acquisition rule is Eq.~\eqref{eq:ha-gate} in the main method section.
The acquisition rule uses a $0.10$ change-probability cutoff and a $2c_t$ group-gain cutoff. These settings define the screening policy; they are not consequences of the calibration theorem. Group screening can exclude individually valuable queries, and any acquisition guarantee still requires its stated coverage assumptions.


\subsection{Matched-Budget Return and Benchmark Breadth}

Table~\ref{tab:headroom-placement} (Section~\ref{sec:headroom-extension}) gives the matched-budget returns and paired differences. These values describe the extension; they do not alter the original matched comparison in Figure~\ref{fig:placement-calibration}(a).

The wider evaluation uses real Qwen2.5-7B advice. Gate success rates are $0.112$ on BabyAI Unlock, $0.096$ on PutNextLocal, $0.249$ on ALFWorld pick-and-place, and $0.279$ on ALFWorld look-at-object. Figure~\ref{fig:nlp-benchmark-breadth} gives the never-query and early-schedule comparisons. These are benchmark-native success fractions, distinct from the original BabyAI return metric and the earlier simulated-advisor ALFWorld study.

\begin{figure*}[!t]
\centering
\includegraphics[width=\textwidth]{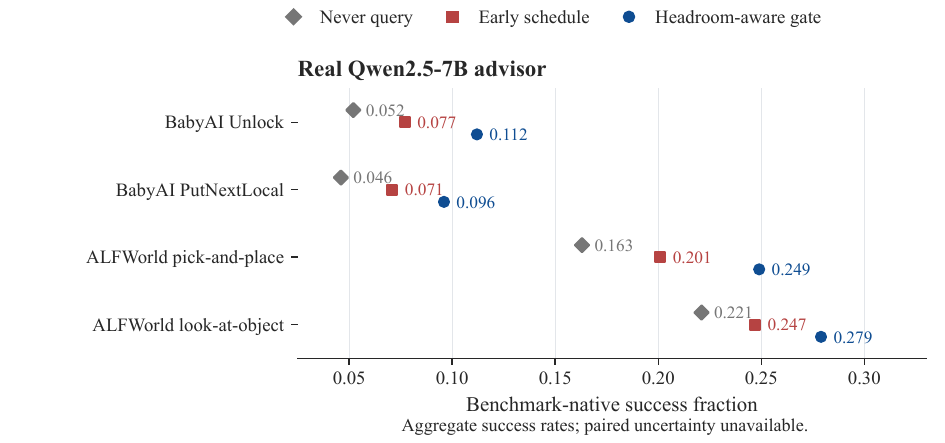}
\caption{Benchmark-native success fractions for the headroom-aware gate with real Qwen2.5-7B advice, compared with never-query and early-schedule policies. Values are reported aggregates without paired uncertainty or adjusted tests.}
\label{fig:nlp-benchmark-breadth}
\end{figure*}

\subsection{Decision-Stratum Coverage}

Each displayed task has $1{,}600$ calibration and $1{,}800$ evaluation opportunities. Decision-stratum coverage is $0.916$, $0.911$, $0.920$, and $0.913$ for GoToObj-1.5B, GoToObj-7B, GoToObjDoor-1.5B, and GoToLocal-1.5B, respectively. Near-threshold coverage ranges from $0.900$ to $0.913$, and worst-stratum coverage from $0.900$ to $0.904$ (Table~\ref{tab:headroom-coverage}).

These held-out proportions do not establish time-uniform or simultaneous coverage, or replace the original controller's undercoverage and vacuous formal radius (Section~\ref{sec:headroom-extension}).

\begin{table*}[t]
\centering\small
\setlength{\tabcolsep}{4pt}
\caption{Decision-aware calibration measurements for the headroom-aware extension. The number of calibration and evaluation opportunities is reported separately for each task. Coverage values describe the indicated held-out subsets; conditional and simultaneous validity require their corresponding sampling assumptions.}
\label{tab:headroom-coverage}
\begin{tabular}{llrrrrrr}
\toprule
Task & Advisor & $n_{\rm cal}$ & $n_{\rm eval}$ & Change \% & Decision cov. & Near-thresh. cov. & Worst-stratum cov.\\
\midrule
GoToObj & 1.5B & $1600$ & $1800$ & $19.3$ & $0.916$ & $0.908$ & $0.901$\\
GoToObj & 7B & $1600$ & $1800$ & $21.1$ & $0.911$ & $0.905$ & $0.900$\\
GoToObjDoor & 1.5B & $1600$ & $1800$ & $17.6$ & $0.920$ & $0.913$ & $0.904$\\
GoToLocal & 1.5B & $1600$ & $1800$ & $2.7$ & $0.913$ & $0.900$ & $0.900$\\
\bottomrule
\end{tabular}
\end{table*}

\subsection{Runtime and Task Success}

The extension reports per-decision gate medians of $6.8$ and $8.1$ ms for the 1.5B and 7B advisors, respectively. Advisor-call medians are $610$ and $1{,}140$ ms. Table~\ref{tab:headroom-runtime} gives median and 95th-percentile timings. These summarize separate inference components. Total wall-time savings additionally depend on how often the gate is evaluated, the complete advisor-call workload, and calibration amortization. Table~\ref{tab:headroom-success} gives task-success fractions, which differ from mean episodic return.

\paragraph{Gate-overhead threshold.}
\label{app:gate-overhead}
Let $N$ be the number of gate evaluations and $q_0-q_1$ the reduction in
advisor requests. Write $g$ for gate wall time per decision, $\tau$ for
per-advisor-call time, and $C_{\rm cal}$ for calibration wall time amortized
over the evaluation horizon. The deployed system is faster than a baseline
with no gate only if
\[
 N g+C_{\rm cal}<(q_0-q_1)\tau.
\]
The original-controller BabyAI tables do not give $g$, $C_{\rm cal}$, FLOPs, or measured
energy, so the observed call reductions alone cannot establish this condition.

\begin{table}[!htbp]
\centering\small
\setlength{\tabcolsep}{3pt}
\caption{Reported gate-computation and advisor latency for the extension: per-decision gating and per-call advisor timings on the declared hardware.}
\label{tab:headroom-runtime}
\begin{tabular}{@{}lrrrr@{}}
\toprule
& \multicolumn{2}{c}{Gate (ms) $\downarrow$} & \multicolumn{2}{c}{Advisor (ms) $\downarrow$}\\
\cmidrule(lr){2-3}\cmidrule(l){4-5}
Advisor & Median & p95 & Median & p95\\
\midrule
Qwen2.5-1.5B & $6.8$ & $13.7$ & $610$ & $840$\\
Qwen2.5-7B & $8.1$ & $17.5$ & $1140$ & $1590$\\
\bottomrule
\end{tabular}
\end{table}

\begin{table}[!htbp]
\centering\small
\setlength{\tabcolsep}{3pt}
\caption{Task-success fractions for the headroom-aware gate. These differ from mean episodic return.}
\label{tab:headroom-success}
\resizebox{\columnwidth}{!}{%
\begin{tabular}{@{}lrrr@{}}
\toprule
Task (advisor) & Gate $\uparrow$ & No advice $\uparrow$ & Difference \\
\midrule
GoToObj (1.5B) & \textbf{0.212} & 0.143 & +0.069 \\
GoToObj (7B) & \textbf{0.218} & 0.143 & +0.075 \\
GoToLocal (1.5B) & \textbf{0.190} & 0.186 & +0.004 \\
\bottomrule
\end{tabular}%
}
\end{table}

\section{Extended Discussion, Limitations, and Broader Impact}\label{app:extended}

This appendix extends the closing discussion, the Limitations section, and the Ethical Considerations of the main text.

\subsection{Discussion}
\label{sec:discussion}

The framework distinguishes response uncertainty, decision uncertainty, and execution safety. Response uncertainty concerns what the advisor may return. Decision uncertainty concerns disagreement among plausible action values. Execution safety concerns whether a returned action passes the comparison against all alternatives. High uncertainty in any one quantity does not by itself imply positive query value.

The response model determines whether advice acts as candidate-generation value or as a Bayesian information update (Appendix~\ref{app:theory}). The distinction matters when model confidence differs from correctness or repeated errors are correlated.

Wide early-learning intervals can reject correct advice under Eq.~\eqref{eq:safe}. Larger tolerance admits more advice but weakens Theorem~\ref{thm:reg}, making acceptance part of the empirical tradeoff.


\subsection{Limitations}
%
\label{sec:limitations}


The certificate requires uniform action-value bounds unavailable from the nonstationary BLL approximation. Recalibration and coverage audits detect failures but cannot establish validity under arbitrary shift.

Acquisition is myopic: calls can change representations, exploration, and later opportunities beyond Eq.~\eqref{eq:true-query-value}'s continuation. The dual threshold does not solve full Bayes-adaptive budgeted control.

The main formulation parses responses into actions. Language models can instead produce plans, subgoals,
reward programs, or explanations. Extending the response space requires
new value estimators and safety tests. Parser failures are included, but semantic errors that map to a valid
yet unintended action remain difficult to detect.

Identical prompts can repeat errors. Learning this dependence requires related-context observations, so queries in new regions may waste budget. Prompt-diversification ablations do not establish independence.

Deployment determines the relevant cost. Token price, local serving latency, energy, and opportunity cost measure different resources. We report calls, tokens, and latency alongside scalarized utility so that each comparison can be interpreted in its stated units.

\subsection{Broader Impact}
%
\label{sec:impact}


Fewer calls can lower monetary, latency, and energy costs; separate execution checks can reduce reliance on misleading advice. Benefits depend on the fallback, parser, bounds, and task specification. The certificate does not establish general alignment or deployment safety.

Cheaper advice can expand autonomous deployment, including harmful or poorly specified uses. Task constraints, audit logs, access control, and human oversight remain necessary. Real-model evaluation uses public benchmarks without sensitive user data.

\end{document}